\pdfoutput=1
\documentclass{article}

\usepackage{ifthen}
\newboolean{istechrpt}
\setboolean{istechrpt}{true}
\newboolean{anonymous}  %
\setboolean{anonymous}{false}
\newcommand{\TechReportOnly}[1]{ \ifthenelse{\boolean{istechrpt}}{#1}{ }}

\usepackage{nips10submit_e, times} %

\nipsfinalcopy

\usepackage[tight]{subfigure}

\usepackage{bbm}

\usepackage{amsmath}
\usepackage{amsthm}
\usepackage{amsfonts}
\usepackage{amssymb}
\usepackage{graphicx}
\usepackage{subfigure}
\usepackage[boxed,noend]{algorithm2e}
\usepackage{xspace}

\usepackage{microtype}
\usepackage{times}
\usepackage{color}

\usepackage[makeroom]{cancel}

\usepackage{stmaryrd}
\newboolean{showcomments}
\setboolean{showcomments}{false}
\ifthenelse{\boolean{istechrpt}}{
}{ }

\newcommand{\algoecd}{{\sc EC$\phantom{}^2$}\xspace}

\newcommand{\algoname}{{\sc EffECXtive}\xspace} 

\newlength{\halfgnat}
\newlength{\gnat}
\newcommand{\ferr}{\btestsval_{\groundset}}

\newcommand{\ODT}{Optimal Decision Tree\xspace}
\newcommand{\ODTs}{ODT\xspace}
\newcommand{\ECD}{Equivalence Class Determination\xspace}
\newcommand{\ECDs}{ECD\xspace}

\newcommand{\fec}{f_{EC}}
\newcommand{\fgbs}{f_{GBS}}

\newcommand{\dec}{d}
\newcommand{\decset}{\mathcal{D}}
\newcommand{\loss}{\ell}

\newcommand{\vectorize}[1]{\ensuremath{\mathbf{#1}}}

\renewcommand{\H}[0]{{\cH}} %
\newcommand{\nclass}{m} %
\newcommand{\nhypo}{n} %
\newcommand{\ndata}{N} %
\newcommand{\ntests}{\ndata} %
\newcommand{\nout}{\ell}%
\newcommand{\pmin}{\ensuremath{p_{\min}}}
\newcommand{\cost}{\ensuremath{c}} %
\newcommand{\support}{\operatorname{supp}}
\newcommand{\payoff}{\ell}
\newcommand{\payoffs}{\mathcal{L}}
\newcommand{\pot}{\Phi}

\newcommand{\edges}{\ensuremath{\mathcal{E}}}
\newcommand{\w}{\ensuremath{w}}
\newcommand{\cut}[2]{\edges_{#1}\paren{#2}}

\newcommand{\htrue}[0]{\ensuremath{h^*}} %

\newcommand{\obs}[0]{\ensuremath{\btestval_{\groundset}}}%

\newcommand{\entropy}[1]{\mathbb{H}\paren{#1}}

\newcommand{\term}[0]{adaptive\xspace}  %
\newcommand{\policy}[0]{\ensuremath{\pi}}
\newcommand{\rlz}[0]{\ensuremath{h}} %
\newcommand{\outcome}[0]{\ensuremath{x}} %
\newcommand{\outcomes}[0]{\ensuremath{\mathcal{X}}} %

\newcommand{\quota}[0]{\ensuremath{Q}} %

\newcommand{\cavgdec}[1]{{\cost}{(#1)}} %

\newcommand{\diff}[3]{\ensuremath{\Delta_{\text{#1}}
    \hspace{-0.3mm}\paren{ {#3}\! \mid \! {#2} }  }}

\newcommand{\prior}[0]{\ensuremath{P}}
\newcommand{\vs}[0]{\ensuremath{\mathcal{V}}\xspace} %
\newcommand{\played}[2]{\groundset({#1}, {#2})}

\newcommand{\indicator}[1]{\ensuremath{\one[{#1}]}} 
\newcommand{\indicatorset}[1]{\ensuremath{\one{\left[{#1}\right] } }}

\newcommand{\daniel}[1]{\ifthenelse{\boolean{showcomments}}{\textcolor{red}{Daniel: #1}}{}}
\newcommand{\andreas}[1]{\ifthenelse{\boolean{showcomments}}{\textcolor{blue}{Andreas: #1}}{}}

\newcommand{\commentout}[1]{}
\newcommand{\ignore}[1]{}

\newlength{\presec}
\newlength{\postsec}
\newlength{\presubsec}
\newlength{\postsubsec}
\newlength{\prepara}
\newlength{\postpara}
\newcommand{\initOneLiners}{%
    \setlength{\itemsep}{0pt}
    \setlength{\parsep }{0pt}
    \setlength{\topsep }{0pt}
}

\newenvironment{proofof}[1]{
\noindent{\bf Proof of {#1}:}}
{\hfill$\blacksquare$
}

\newcommand{\paren} [1] {\ensuremath{ \left( {#1} \right) }}

\newcommand{\set}[1]{\ensuremath{\left\{#1\right\}}}

\newcommand{\prob}[1]{\ensuremath{\mathbb{P}\left[#1\right] }}

\newcommand{\expct}[1]{\mathbb{E}\left[#1\right]}
\newcommand{\expctover}[2]{\mathbb{E}_{#1}\!\left[#2\right]}
\def \argmax {\mathop{\rm arg\,max}}
\def \argmin {\mathop{\rm arg\,min}}

\newcommand{\dom}[0]{\operatorname{dom}}
\renewcommand{\implies}[0]{\ensuremath{\Rightarrow}}

\newcommand{\NonNegativeReals}{\ensuremath{\mathbb{R}_{\ge 0}}}

\newcommand{\nats}{\ensuremath{\mathbb{N}}}
\newcommand{\reals}{\ensuremath{\mathbb{R}}}
\newcommand{\rationals}{\ensuremath{\mathbb{Q}}}

\newcommand{\hypotheses}[0]{\mathcal{H}}
\newcommand{\hvar}[0]{\ensuremath{H}}
\newcommand{\hypothesis}[0]{\ensuremath{h}}

\newcommand{\tests}{\ensuremath{\mathcal{T}}}
\newcommand {\groundset}{\tests} %
\newcommand{\test}[0]{\ensuremath{t}}
\newcommand{\prlzvec}[1]{\btestval_{#1}}
\newcommand{\rlzvec}[0]{\btestval_{\data}}

\newcommand{\testvec}[1]{\ensuremath{\mathbf{X}_{#1}}}

\newcommand{\tested}[1]{\mathbf{X}_{#1}}
\newcommand{\seen}[1]{\mathbf{x}_{#1}}
\newcommand{\igpolicy}[0]{\policy_{\text{IG}}}

\newcommand{\voipolicy}[0]{\policy_{\text{VoI}}}

\newcommand{\badtassdelta}[0]{\Delta_{\text{Eff}}}

\newcommand{\testvar}[0]{\ensuremath{X}}
\newcommand{\btestvar}[0]{\mathbf{X}}
\newcommand{\btestsvar}[0]{\mathbf{X}}
\newcommand{\testval}[0]{\ensuremath{x}}
\newcommand{\btestval}[0]{\ensuremath{\mathbf{x}}}
\newcommand{\btestsval}[0]{\ensuremath{\mathbf{x}}}

\newcommand{\data}[0]{\tests}

\newcommand{\event}[0]{{\Lambda}}

\newcommand{\class} [1] {\textrm{#1}} %

\newcommand{\NP} {\class{NP}}

\newcommand{\one}{\mathbf{1}}

\newcommand{\cA}{{\mathcal{A}}}

\newcommand{\cV}{{\mathcal{V}}}
\newcommand{\cB}{{\mathcal{B}}}

\newcommand{\cH}{{\mathcal{H}}}

\newcommand{\cO}{{\mathcal{O}}}

\newtheorem{theorem}{Theorem}%

\newtheorem{lemma}[theorem]{Lemma}
\newtheorem{proposition}[theorem]{Proposition}

\theoremstyle{definition}

\theoremstyle{remark}

\numberwithin{equation}{section}

\newcommand{\figref}[1]{Fig.~\ref{#1}}
\newcommand{\eqnref}[1]{Eq.~(\ref{#1})}
\newcommand{\secref}[1]{\S\ref{#1}}
\newcommand{\thmref}[1]{Theorem~\ref{#1}}

\newcommand{\propref}[1]{Proposition~\ref{#1}}

\newcommand{\algref}[1]{Algorithm~\ref{#1}}

\begin{document}

\ifthenelse{\boolean{istechrpt}}{
   \title{Near--Optimal Bayesian Active Learning \\with Noisy Observations}
} { %
   \title{Near--Optimal Bayesian Active Learning \\with Noisy Observations}
}

\ifthenelse{\boolean{anonymous}}{ }{
   \author{
     Daniel Golovin\\
     Caltech%
     \And
     Andreas Krause \\                    
     Caltech%
     \And
     Debajyoti Ray \\                    
     Caltech%
   }
   }

\maketitle

\ifthenelse{\boolean{showcomments}}{
\noindent {\color{red} {\Large Draft: Do Not Distribute. (\today)}\\  
Comments are on.  To turn them off, toggle boolean
``showcomments''.}\\ 
}{ }

\begin{abstract}
\looseness -1 We tackle the fundamental problem of Bayesian active learning with noise, where we need to adaptively select from a number of expensive tests in order to identify an unknown hypothesis sampled from a known prior distribution. In the case of noise--free observations, a greedy algorithm called generalized binary search (GBS) is known to perform near--optimally.  We show that if the observations are noisy, perhaps surprisingly, GBS can perform very poorly. 
We develop \algoecd, a novel, greedy active learning algorithm and prove that it is competitive with the optimal policy, thus obtaining the first competitiveness guarantees for Bayesian active learning with noisy observations. Our bounds rely on a recently discovered diminishing returns property called adaptive submodularity, generalizing the classical notion of submodular set functions to adaptive policies.  Our results hold even if the tests have non--uniform cost and their noise is correlated.  
We also propose \algoname, a particularly fast approximation of
\algoecd, and evaluate it on a Bayesian experimental design problem involving human subjects, intended to tease apart competing economic theories of how people make decisions under uncertainty.
\end{abstract}

\ignore{
We tackle the
fundamental problem of Bayesian active learning with noise, where we
need to adaptively select from a number of expensive tests in order
to identify an unknown hypothesis sampled from a known prior
distribution. In the case of noise-free observations, a greedy
algorithm called generalized binary search (GBS) is known to perform
near-optimally. We show that if the observations are noisy, perhaps
surprisingly, GBS can perform very poorly. We develop EC2, a novel,
greedy active learning algorithm and prove that it is competitive
with the optimal policy, thus obtaining the first competitiveness
guarantees for Bayesian active learning with noisy observations. Our
bounds rely on a recently discovered diminishing returns property
called adaptive submodularity, generalizing the classical notion of
submodular set functions to adaptive policies. Our results hold even
if the tests have non–uniform cost and their noise is correlated.
We also propose EffECXtive, a particularly fast approximation of
EC2, and evaluate it on a Bayesian experimental design problem
involving human subjects, intended to tease apart competing economic
theories of how people make decisions under uncertainty.
} %

\vspace{\gnat}
\section{Introduction} \label{sec:intro}\vspace{\gnat}
\looseness -1  
How should we perform experiments to determine the most accurate scientific theory among competing candidates, or choose among expensive medical procedures to accurately determine a patient's condition, or select which labels to obtain in order to determine the hypothesis that minimizes generalization error?  
In all these
applications, we have to sequentially select among a set of noisy,
expensive observations (outcomes of experiments, medical tests, expert
labels) in order to determine which hypothesis (theory, diagnosis, classifier) is
most accurate. This fundamental problem has been studied in a number
of areas, including statistics \cite{Lindley:1956}, decision theory \cite{howard66voi}, machine learning \cite{nowak09,dasgupta04} %
and others.  One way to formalize such active
learning problems is  \emph{Bayesian experimental design} \cite{Chaloner1995}, where one
assumes a prior on the hypotheses, as well as probabilistic
assumptions on the outcomes of tests. The goal then is to determine
the correct hypothesis while minimizing the cost of the
experimentation. Unfortunately, finding this optimal policy is not
just NP-hard, but also hard to approximate \cite{chakaravarthy07decision}.  Several heuristic
approaches have been proposed that perform well in some applications,
but do not carry theoretical guarantees (e.g., \cite{MacKay1992}). In the case where
observations are \emph{noise-free}\footnote{This case is known as the
  \emph{Optimal Decision Tree} (ODT) problem.}, a simple algorithm, 
\emph{generalized binary search}\footnote{GBS greedily selects tests 
to maximize, in expectation over the test outcomes, the prior probability
mass of eliminated hypotheses (i.e., those with zero posterior
probability, computed w.r.t.~the observed test outcomes).}\emph{(GBS)}
run on a 
modified prior, is guaranteed to be competitive with the optimal policy; 
 the expected number of queries is a factor of $O(\log n)$ (where $n$ is the number of hypotheses) more than that of the optimal policy~\cite{kosaraju99}, which matches lower bounds up to constant factors~\cite{chakaravarthy07decision}.

The important case of \emph{noisy} observations, however, as present in most applications, is much less well understood. While there are some recent positive results in understanding the label complexity of noisy active learning \cite{nowak09,balcan06agnostic}, 
to our knowledge, so far there are no algorithms that are provably
competitive with the optimal sequential policy, except in very
restricted settings \cite{krause09}. In this paper, we introduce a
general formulation of Bayesian active learning with noisy
observations that we call the \emph{Equivalence Class Determination}
problem. We show that, perhaps surprisingly,  generalized binary
search performs poorly in this setting, as do greedily (myopically) maximizing the
information gain (measured w.r.t. the distribution on equivalence
classes) or the  decision-theoretic value
of information.
This motivates us to introduce a novel active learning criterion, and use it 
to develop a greedy active learning algorithm called the \textbf{E}quivalence
\textbf{C}lass \textbf{E}dge \textbf{C}utting algorithm (\algoecd),
whose expected cost is competitive to that of
the optimal policy. 
Our key insight is that our new objective function satisfies \emph{adaptive submodularity} \cite{golovin10adaptive}, a natural diminishing returns property that generalizes the classical notion of submodularity to adaptive policies. 
Our results also allow us to relax the common
assumption that the outcomes of the tests are conditionally
independent given the true hypothesis. We also develop the
\textbf{Eff}icient \textbf{E}dge \textbf{C}utting appro\textbf{X}imate
objec\textbf{tive} algorithm (\algoname), 
an efficient approximation to \algoecd, and evaluate it on a Bayesian
experimental design problem intended to tease apart competing theories
on how people make decisions under uncertainty, including Expected
Value \cite{vonneumann47}, Prospect Theory \cite{kahneman79},
Mean-Variance-Skewness \cite{hanoch70} and Constant Relative Risk
Aversion \cite{pratt64}. In our experiments, \algoname typically
outperforms existing experimental design criteria such as information
gain, uncertainty sampling, GBS, and decision-theoretic value of information.  Our results from human subject experiments further reveal that \algoname can be used as a real-time tool to classify people according to the economic theory that best describes their behaviour in financial decision-making, and reveal some interesting heterogeneity in the population.

\vspace{\gnat}
\section{Bayesian Active Learning in the Noiseless Case}\vspace{\gnat}
In the Bayesian active learning problem, we would like to distinguish
among a given set of hypotheses
$\hypotheses=\set{\hypothesis_{1},\dots,\hypothesis_{\nhypo}}$ by
performing tests from a set $\tests=\set{{1}, \dots, {\ntests}}$ of
possible tests. Running test $\test$ incurs a cost of $\cost(\test)$ and produces
an outcome from a finite set of
outcomes $\outcomes = \set{1, 2, \ldots, \nout}$.
We let $\hvar$ denote the random variable which equals the true
hypothesis, and model the outcome of each test $\test$ by a random variable
$\testvar_\test$ taking values in $\outcomes$.  
We denote the observed outcome of test $\test$ by $\testval_{\test}$.
We further suppose we have a prior distribution $\prior$ modeling our assumptions on the
joint probability $\prior(\hvar, \testvar_1, \ldots,
\testvar_{\ntests})$ over the hypotheses and  test outcomes.
In the noiseless case, we assume that the outcome
of each test is deterministic given the true hypothesis, i.e., for
each $\hypothesis \in \hypotheses$, 
$\prior(\testvar_1, \ldots, \testvar_{\ntests} \mid \hvar = \hypothesis)$
is a deterministic distribution.  Thus, each hypothesis
$\hypothesis$ is associated with a particular vector of test outcomes.
We assume, w.l.o.g.,
that no two hypotheses lead to the same outcomes for all tests. Thus,
if we perform all tests, we can uniquely determine the true
hypothesis. 
However in most applications we will wish to avoid performing every
possible test, as this is prohibitively expensive.  
Our goal is to find an adaptive policy for running tests that allows
us to determine the value of $\hvar$ while minimizing
the cost of the tests performed.  Formally, a policy $\policy$
(also called a conditional plan) is a partial mapping $\policy$
from partial observation vectors $\prlzvec{\cA}$ to tests, specifying which test to run next (or
whether we should stop testing) for any observation vector
$\prlzvec{\cA}$. Hereby, $\prlzvec{\cA} \in \outcomes^{\cA}$ is a vector of
outcomes indexed by a set of tests $\cA \subseteq \tests$ that we have
performed so far~\footnote{Formally we also require that $(\testval_{\test})_{\test \in \cB} \in \dom(\policy)$ and $\cA\subseteq \cB$,
  implies $(\testval_{\test})_{\test \in \cA} \in\dom(\policy)$ (c.f., \cite{golovin10adaptive}).}
(e.g., the set of labeled examples in active learning, or outcomes of a
set of medical tests that we ran).  After
having made observations $\prlzvec{\cA}$, we can rule out
inconsistent hypotheses. We denote the set of hypotheses consistent
with event $\event$
(often called the \emph{version space} associated with $\event$) by
$\vs(\event):=\set{\hypothesis\in\hypotheses: \prior(\hypothesis
\mid \event) > 0}$.  
We call a
policy \emph{feasible} if it is guaranteed to uniquely determine the
correct hypothesis. That is, upon termination with observation
$\prlzvec{\cA}$, it must hold that $|\vs(\prlzvec{\cA})|=1$. We can define the expected
cost of a policy $\policy$ by
\vspace{\halfgnat}
$$\cost(\policy):=\sum_{\hypothesis}\prior(\hypothesis) \cost(\played{\policy}{\hypothesis}) \vspace{\gnat}$$
where $\played{\policy}{\hypothesis}\subseteq\tests$ is the set of
tests run by policy $\policy$ in case $\hvar = \hypothesis$.
Our goal is to find a feasible policy $\policy^{*}$ of minimum
expected cost, i.e.,
\vspace{-1mm}
\begin{equation}
\policy^{*}=\argmin \set{\cost(\policy) : \policy \text{ is feasible}}
\label{eq:optdt}
\vspace{-1mm}
\end{equation}
A policy $\pi$ can be naturally represented as a decision tree $T^{\policy}$, and thus problem~\eqref{eq:optdt} is often called the \emph{Optimal Decision Tree} (\ODTs) problem.

Unfortunately,
obtaining an approximate policy $\policy$ for which
$\cost(\policy)\leq \cost(\policy^{*}) \cdot o(\log(n))$ is 
$\NP$-hard~\cite{chakaravarthy07decision}.
Hence, various heuristics are employed to solve the \ODT problem and
its variants.
Two of the most popular heuristics are to select tests greedily to
maximize the \emph{information gain} (IG) conditioned on previous test
outcomes, and \emph{generalized binary search} (GBS).
Both heuristics are greedy, and after having made observations 
$\prlzvec{\cA}$ will select 
$$\test^{*}=\argmax_{\test\in\tests} \diff{Alg}{\prlzvec{\cA}}{\test} /{\cost(\test)},$$
\looseness -1 where $\text{Alg} \in \set{\text{IG},\text{GBS}}$.  
Here,
$\diff{IG}{\prlzvec{\cA}}{\test} := \entropy{\tested{\groundset} \mid \seen{\cA}} -
\expctover{\testval_{\test} \sim \testvar_{\test} \mid \seen{\cA}}{\entropy{\tested{\groundset} |
    \btestsval_{\cA}, \testval_{\test}}}$ is the marginal information
gain measured with respect to the Shannon entropy
$\entropy{\btestsvar} := \expctover{\btestsval}{ - \log_2
  \prior(\btestsval)}$, and 
$\diff{GBS}{\prlzvec{\cA}}{\test} :=\prior(\vs(\prlzvec{\cA}))-\sum_{\outcome\in\outcomes}
\prior(\testvar_{\test} = \outcome\mid\prlzvec{\cA})
\prior(\vs(\prlzvec{\cA}, \testvar_\test =\outcome))$
is the expected
reduction in version space probability mass.
Thus, both heuristics greedily
chooses the test that maximizes the benefit-cost ratio, measured with
respect to their particular benefit functions.
They stop after running a set of tests $\cA$ such that
$|\vs(\prlzvec{\cA})|=1$, i.e., once the true hypothesis has been
uniquely determined.

It turns out that for the (noiseless) \ODT problem, these two
heuristics are equivalent ~\cite{zheng05}, as can be proved using the chain rule of
entropy.
\daniel{Give a proof in the appendix?  Include Zheng et al reference?}
Interestingly, despite its myopic nature GBS has been shown~\cite{kosaraju99,dasgupta04,guillory09,golovin10adaptive} to obtain
near-optimal expected cost: the strongest known bound is 
$\cost(\policy_{GBS})\le\cost(\policy^{*})\paren{\ln (1/\pmin)+1}$ 
where $\pmin :=\min_{\hypothesis\in\hypotheses} \prior(\hypothesis)$.
Let $\prlzvec{S}(\hypothesis)$ be the unique vector $\prlzvec{S} \in \outcomes^{S}$
such
that $\prior(\prlzvec{S} \mid \hypothesis) = 1$.
The result above is proved by exploiting the
fact that 
$\fgbs(S,\hypothesis):=1-\prior(\vs(\prlzvec{S}(\hypothesis)))
+\prior(\hypothesis)$
is \emph{adaptive
  submodular} and \emph{strongly adaptively
  monotone}~\cite{golovin10adaptive}. 
Call $\prlzvec{\cA}$ a \emph{subvector} of $\prlzvec{\cB}$ if $\cA
\subseteq \cB$ and $\prior(\prlzvec{\cB} \mid \prlzvec{\cA}) > 0$.  
In this case we write $\prlzvec{\cA} \prec \prlzvec{\cB}$.
A function $f:2^{\data}\times \hypotheses$ is called adaptive submodular w.r.t. a distribution $\prior$, if for any  $\prlzvec{\cA}\prec\prlzvec{\cB}$ and any test $\test$ it holds that 
$\diff{}{\prlzvec{\cA}}{\test} \geq \diff{}{\prlzvec{\cB}}{\test}$, where 
$$\diff{}{\prlzvec{\cA}}{\test} :=\expctover{\hvar}{f(\cA \cup \set{\test}, \hvar) - f(\cA, \hvar)\ \mid \prlzvec{\cA}}.$$
Thus, $f$ is adaptive submodular if the expected marginal benefits
$\diff{}{\prlzvec{\cA}}{\test}$ of adding a new test $\test$ can only
decrease as we gather more observations. $f$ is called \emph{strongly
  adaptively monotone} w.r.t.~$\prior$ if, informally, ``observations never hurt'' with respect to the expected reward. Formally, for
all $\cA$, all $\test \notin \cA$, and all
$\outcome \in \outcomes$ 
we require 
$\expctover{\hvar}{f(\cA, \hvar) \ \mid \  \prlzvec{\cA}} \le 
\expctover{\hvar}{f(\cA \cup \set{\test}, \hvar) \ \mid \ 
  \prlzvec{\cA}, \testvar_{\test} = \outcome}\mbox{.} $ %

The performance guarantee for GBS follows from the following general result about the greedy algorithm for adaptive submodular functions (applied with $Q=1$ and $\eta=\pmin$):
\begin{theorem}[Theorem $10$ of \cite{golovin10adaptive} with
  $\alpha
  = 1$] \label{thm:min-set-cover-avg-generalized-with-costs}
Suppose $f:2^{\data} \times \hypotheses \to
\NonNegativeReals$ is \term submodular and strongly adaptively
monotone with respect to $\prior$
and there exists $Q$ such that 
$f(\groundset, \rlz) = Q$ for all $\rlz$.
Let $\eta$ be any value such that 
$f(S, \rlz) > Q - \eta$ implies $f(S, \rlz) = Q$ for all sets  $S$
and hypotheses $\rlz$.
Then for self--certifying instances the adaptive greedy policy
$\policy$ satisfies 
$ \cost({\policy}) \le  \,
\cost({\policy^*})\paren{\ln \paren{\frac{Q}{\eta}} + 1} 
\mbox{.} $
\end{theorem}

The technical requirement that instances be \emph{self--certifying}
means that the policy will have proof that it has obtained the maximum
possible objective value, $Q$, immediately upon doing so.
 It is not
difficult to show that this is the case with the instances we consider
in this paper.
We refer the interested reader to~\cite{golovin10adaptive} for more detail.

In the following sections, we will use the concept of adaptive
submodularity to provide the first approximation guarantees for
Bayesian active learning with noisy observations.

\vspace{\gnat} \section{The \ECD Problem and the \algoecd Algorithm}\vspace{\gnat}
\label{sec:equiv-class-determination}

\looseness -1 We now wish to consider the Bayesian active learning problem where tests can have noisy outcomes. 
Our general strategy is to reduce the problem of noisy observations to the noiseless setting. To gain intuition, consider a simple model where tests have binary outcomes, and we know that the outcome of exactly one test, chosen uniformly at random unbeknown to us, is flipped. If any pair of hypotheses $\hypothesis\neq\hypothesis'$ differs by the outcome of at least three tests, we can still uniquely determine the correct hypothesis after running all tests. In this case we can reduce the noisy active learning problem to the noiseless setting by, for each hypothesis, creating $\ndata$ ``noisy'' copies, each obtained by flipping the outcome of one of the $\ndata$ tests. The modified prior $\prior'$ would then assign mass $\prior'(\hypothesis')=\prior(\hypothesis)/\ndata$ to each noisy copy $\hypothesis'$ of $\hypothesis$. The conditional distribution $\prior'(\testvec{\data}\mid\hypothesis')$ is still deterministic (obtained by flipping the outcome of one of the tests).
 Thus, each hypothesis $\hypothesis_{i}$ in the original problem is now associated with a set $\H_{i}$ of hypotheses in the modified problem instance.
However, instead of selecting tests to determine which noisy copy has been realized, we only care which set $\H_{i}$ is realized. 

\paragraph{The \ECD problem (\ECDs).}  \looseness -1 More generally, we introduce the \emph{\ECD
  problem}\footnote{Bellala et
  al. simultaneously studied ECD~\cite{bellala10extensions},
  and, like us, used it to model active learning with noise~\cite{bellala09}.  
  They  developed  an extension of GBS for ECD.
  We defer a detailed comparison of our approaches to future work.},
where our set of hypotheses $\hypotheses$ is partitioned
into a set of $\nclass$
equivalence classes $\set{\H_1, \ldots, \H_\nclass}$ so that $\hypotheses = \biguplus_{i=1}^{\nclass} \H_i $, and the goal is to determine which class $\H_i$ the true hypothesis lies in.
Formally, upon termination with observations $\prlzvec{\cA}$ we require
that $\vs(\prlzvec{\cA})\subseteq \H_{i}$ for some $i$.
As with the \ODTs problem, the goal is to minimize the expected cost
of the tests, where the expectation is taken over the true
hypothesis sampled from $\prior$.  In~\secref{sec:noise}, we will show
how the \ECD problem arises naturally from Bayesian experimental
design problems in probabilistic models.

Given the fact that GBS performs near-optimally on the \ODT problem, a natural approach to solving \ECDs would be to run GBS until the termination condition is met. Unfortunately, and perhaps surprisingly, GBS can perform very poorly on the \ECDs problem.
Consider an instance with a uniform prior over $\nhypo$ hypotheses, $\hypothesis_1,
\ldots, \hypothesis_\nhypo$, and two equivalence classes $\H_1 := \set{h_i : 1 \le i <
  n}$ and $\H_2 := \set{h_n}$.  There are tests $\tests = \set{1,
  \ldots, n}$ such that $h_i(\test) = \indicatorset{i = \test}$, all of
unit cost. Hereby, $\indicatorset{\event}$ is the indicator 
variable for event $\event$.
In this case, the optimal policy  only needs to select test $n$,
however GBS  may 
select tests $1, 2, \ldots, n$ in order until running test $\test$, where 
$\hvar = h_\test$ is the true hypothesis.  Given our uniform prior, it takes 
$n/2$ tests in expectation until this happens, so that GBS pays, in expectation, $n/2$ times the optimal
expected cost in this instance.

The poor performance of GBS in this instance may be attributed to its
lack of consideration for the equivalence classes.  Another natural heuristic
would be to run the greedy information gain policy, only
with the entropy measured with respect to the probability distribution
on \emph{equivalence classes} rather than hypotheses.  Call this policy
$\igpolicy$.  It is clearly aware of the equivalence classes, as it 
adaptively and myopically selects tests to reduce the uncertainty of
the realized class, measured w.r.t. the Shannon entropy.  However, we can show there are
instances in which it pays $\Omega(n/\log(n))$ times the
optimal cost, even under a uniform prior.  
\ifthenelse{\boolean{istechrpt}}{
Refer to Appendix~\ref{sec:infogain-stinks} for details.
}{ 
See the long version of this paper~\cite{dtree-arxiv-version} for details.
}

\daniel{Mention negative result for Bayesian decision-theoretic value
  of information criteria.  Perhaps expand out to a theorem statement here?}

\paragraph{The \algoecd algorithm.} 
\looseness -1 The reason why 
GBS fails is because reducing the version space mass
does not necessarily facilitate differentiation among the classes
$\H_{i}$.  The reason why $\igpolicy$ fails is that there are
complementarities among tests; a set of tests can be far better than
the sum of its parts.
Thus, we would like to optimize an objective function that 
encourages differentiation among classes, but lacks complementarities.
We adopt a very elegant idea from Dasgupta~\cite{dasgupta05}, 
and define weighted edges between
hypotheses that we aim to distinguish between.  However, instead of
introducing edges between arbitrary pairs of hypotheses (as done in
\cite{dasgupta05}), we only introduce edges between hypotheses in
different classes.  Tests will allow us to cut 
edges inconsistent with their outcomes, and we aim to eliminate all inconsistent edges while minimizing
the expected cost incurred.  We now formalize this intuition.

Specifically, we define a set of edges 
$\edges = \cup_{1 \le i < j \le \nclass} \set{\set{h, h'} : h \in \H_i, h' \in \H_j }$,
consisting of all (unordered) pairs of hypotheses belonging to
distinct classes.
These are the edges that must be \emph{cut}, by which we mean for any 
edge $\set{h, h'} \in \edges$, at least one hypothesis in $\set{h,
  h'}$ must be ruled out (i.e., eliminated from the version space).
Hence, a test $\test$ run under true hypothesis $\rlz$
is said to cut edges $\cut{\test}{\rlz} := \set{\set{h', h''} : h'(\test) \neq h(\test) \text{ or } h''(\test) \neq h(\test) }$. See \figref{fig:edge-cutting} for an illustration.
We define a weight function $\w:\edges \to \NonNegativeReals$ by 
$\w(\set{h, h'}) := \prior(h) \cdot \prior(h')$. 
We extend the weight function to an additive (modular) function on sets of edges
in the natural manner, i.e.,  $\w(\edges') := \sum_{e \in \edges'} \w(e)$.
The objective $\fec$ that we will greedily maximize is then defined as
the weight of the edges cut (EC):
\begin{equation}
  \label{eq:class1}
  \fec(\cA, \rlz) := \w\Bigl(\bigcup_{\test \in \cA} \cut{\test}{\rlz}
  \Bigr)
\vspace{\gnat}
\vspace{\gnat}
\end{equation}

The key insight that allows us to prove approximation guarantees for
$\fec$ is that $\fec$ shares the same beneficial properties that make
$\fgbs$ amenable to efficient greedy optimization.
\ifthenelse{\boolean{istechrpt}}{
We prove this fact, as stated in \propref{prop:ecprop}, in Appendix~\ref{sec:proofs}.
}{ 
The proof of this fact, as stated in \propref{prop:ecprop}, can be found in the long version of this paper~\cite{dtree-arxiv-version}.
}

\begin{proposition}\label{prop:ecprop}
The objective $\fec$ is strongly adaptively monotone and adaptively submodular. 
\end{proposition}

Based on the objective $\fec$, we can calculate the marginal benefits for test $\test$ upon observations $\prlzvec{\cA}$ as
$$\diff{EC}{\prlzvec{\cA}}{\test} :=\expctover{\hvar}{\fec(\cA \cup \set{\test}, \hvar) - \fec(\cA, \hvar)\ \mid \prlzvec{\cA}}.$$
We call the adaptive policy $\policy_{EC}$ that, after observing
$\prlzvec{\cA}$, greedily selects test \linebreak
$\test^{*}\in\argmax_{\test} \diff{EC}{\prlzvec{\cA}}{\test}/\cost(\test)$, the \algoecd algorithm (for \emph{equivalence class edge cutting}).

Note that these instances are self--certifying, because we obtain
maximum objective value 
if and only if the version space lies within an equivalence class, and the
policy can certify this condition when it holds.
So we can apply
\thmref{thm:min-set-cover-avg-generalized-with-costs} to show \algoecd obtains a 
$\ln(\quota /\eta)+1$ approximation to \ECD. Hereby,  $\quota =
\w(\edges) = 1 - \sum_{i} (\prior(\rlz \in \H_i))^2 \le 1$ is the total weight of all edges that need to be cut, and $\eta = \min_{e \in \edges} \w(e) \ge \pmin^2$ is a bound on the minimum weight among all edges. We have the following result:
\begin{figure}[t]
\centering \subfigure[\emph{The \ECD problem}]{
\includegraphics[width=0.55\textwidth]{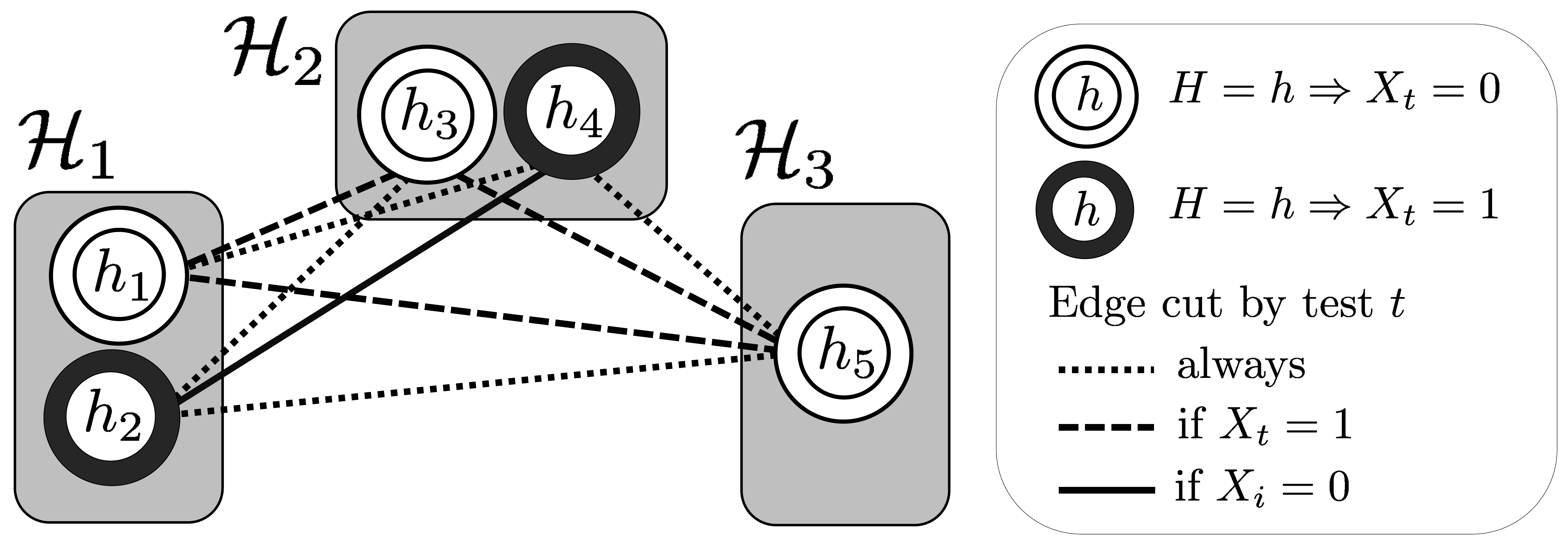}
 \label{fig:edge-cutting}
 }
\subfigure[\emph{Error model}]{
\includegraphics[width=0.36\textwidth]{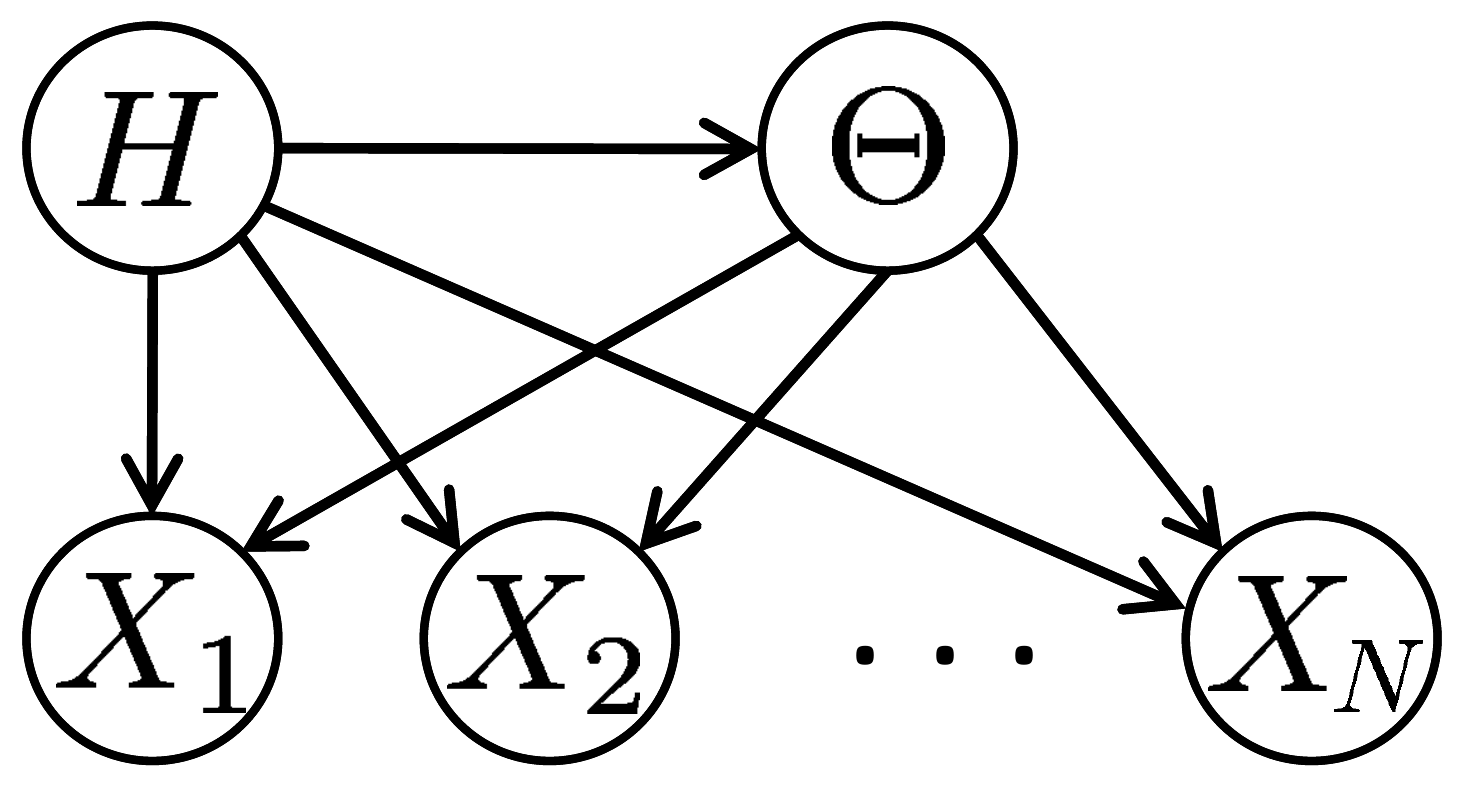}
 \label{fig:error-model}
 }
 \vspace{-3mm}
 \caption{\footnotesize (a) An instance of \ECD with binary test
          outcomes, shown with the set
          of edges that must be cut, and depicting the effects of
          test $i$ under different outcomes. (b) The graphical model underlying our error model.
 }\vspace{-3mm}
\end{figure}

 \begin{theorem} \label{thm:equiv-class-determination}
Suppose $\prior(h)$ is rational for all $h \in \hypotheses$.
For the  adaptive greedy policy $\policy_{EC}$ implemented by \algoecd it holds that
$$\cost(\policy_{EC})\leq (2\ln(1/\pmin)+1)\cost(\policy^{*}),$$
where $\pmin := \min_{h \in \hypotheses} \prior(h)$ is the minimum prior probability of any hypothesis, and $\policy^{*}$ is the optimal policy for the \ECD problem.
 \end{theorem}
In the case of unit cost tests, we can apply a technique of Kosaraju
et al.~\cite{kosaraju99}, originally developed for the GBS algorithm,
to improve the approximation guarantee to $\cO(\log \nhypo)$ by
applying \algoecd with a modified prior distribution.  We defer details
to the full version of this paper.

\TechReportOnly{
\paragraph{A Fast Implementation of \algoecd.}
The time running time of \algoecd is
dominated by the evaluations of $\diff{EC}{\prlzvec{\cA}}{\test}$.
The naive way to compute $\diff{EC}{\prlzvec{\cA}}{\test}$
is to construct a graph on the $\nhypo$ hypotheses with weighted
edges as prescribed by \algoecd, and then see which edges are cut by
$\test$ for each potential test outcome.  Assuming there are $\nout$ possible outcomes of a test, and
that we can evaluate $\hypothesis(\test)$ in unit time
for all $\hypothesis$ and $\test$, this will take $\cO(\nhypo^2
\nout)$ time.  With $\ntests$ tests, the total time per round of \algoecd 
is $\cO(\ntests \nhypo^2 \nout)$.
However, there is a much faster way to compute
$\diff{EC}{\prlzvec{\cA}}{\test}$.
Note that $\diff{EC}{\prlzvec{\cA}}{\test}$ equals
\[
\expctover{\testval_{\test}  \sim
  \testvar_{\test} \mid \seen{\cA}}{ 
\frac{1}{2}\sum_{i \neq j} \Big( 
\prior\paren{\H_i \cap
  \vs(\prlzvec{\cA})}\prior\paren{\H_j \cap
  \vs(\prlzvec{\cA})} - 
\prior\paren{\H_i \cap
  \vs(\prlzvec{\cA}, \testval_{\test})} \prior\paren{\H_j \cap
  \vs(\prlzvec{\cA}, \testval_{\test})} \Big)}.
\]
Now, 
compute $\alpha(i,\testval_{\test}) := \prior\paren{\H_i \cap
  \vs(\prlzvec{\cA}, \testval_{\test})}$ for each $i$ and $\testval$,
then compute $\beta(i) := \prior\paren{\H_i} = \sum_{\testval}
\alpha(i,\testval)$.
Next, compute $\gamma(\testval_{\test}) := \prior\paren{ \testval_{\test}  \mid \prlzvec{\cA}} =
\sum_{i} \alpha(i, \testval_{\test} ) / \sum_{i} \beta(i)$.
All of these terms can be computed in total time $\cO(\nhypo)$
by iterating over the hypotheses and for each $\hypothesis$ adding
$\prior(\hypothesis)$ to the appropriate terms (i.e., $\beta(i)$,
$\alpha(i,\testval)$, and $\gamma(\testval)$ if $\hypothesis
\in \H_i$ and $\hypothesis(t) = \testval$). 
Using these variables, we can rewrite 
$\diff{EC}{\prlzvec{\cA}}{\test}$ as 
$\expctover{\testval_{\test}  \sim
   \testvar_{\test} \mid \seen{\cA}}{ 
 \frac{1}{2}\sum_{i \neq j} \Big( \beta(i) \beta(j) - \alpha(i,
   \testval_{\test} ) \alpha(j,
   \testval_{\test} ) \Big)}$.
Note that for any $\eta_1, \eta_2, \ldots, \eta_\nclass
\in \reals$, we have 
$\sum_{i \neq j} \eta_i \eta_j = \paren{\sum_{i} \eta_i }^2 -
\sum_{i} \eta_i^2$.
Using this equality, we can evaluate sums like 
$\sum_{i \neq j} \Big( \beta(i) \beta(j) - \alpha(i,
  \testval_{\test} ) \alpha(j,\testval_{\test} ) \Big)$ in 
$\cO(\nclass)$ time, where there are $\nclass$ equivalence classes.
Hence the total time to evaluate $\diff{EC}{\prlzvec{\cA}}{\test}$ is 
$\cO(\nhypo + \nclass \nout)$ using this method.
In a similar manner, we can reduce the running time still further to
$\cO(\nhypo)$, by incrementally computing terms such as 
$\paren{\sum_{i} \alpha(i,\testval)}^2$ and 
$\sum_{i} \alpha(i,\testval)^2$ as we iterate through the hypotheses.
The total time per round of \algoecd 
is then $\cO\paren{\ntests \nhypo}$.
Additionally, the number of evaluations the algorithm needs to make
can often be significantly reduced in practice using the
\emph{accelerated adaptive greedy} algorithm, as discussed
in~\cite{golovin10adaptive}.
} %

\vspace{\gnat}\section{Bayesian Active Learning with Noise and the \algoname Algorithm}\vspace{\gnat}
\label{sec:noise}

We now address the case of noisy observations, using ideas
from~\secref{sec:equiv-class-determination}.
With noisy observations, the conditional distribution
$P(\testvar_{1},\dots,\testvar_{\ndata}\mid\hypothesis)$ is no longer
deterministic.  We model the noise using 
an additional random variable $\Theta$.
\figref{fig:error-model} depicts the underlying graphical model.
The vector of test outcomes $\btestsval_{\groundset}$ is assumed to be
an arbitrary, deterministic
function $\ferr:\hypotheses \times \support(\Theta)  \to
\outcomes^\ntests$; hence 
$\btestvar_{\groundset} \mid \hypothesis$ is distributed as 
$\ferr(\hypothesis, \Theta_{\hypothesis})$ where $\Theta_{\hypothesis}$ is distributed as 
$\prior(\theta \mid \hypothesis)$.
For example, there might be up to $s = |\support(\Theta)|$
ways any particular disease could manifest itself, with different
patients with the same disease suffering from different
symptoms.

In cases where it is always possible to identify the true 
hypothesis, i.e., $\ferr(\hypothesis, \theta) \neq \ferr(\hypothesis',
\theta')$ for all $\hypothesis \neq \hypothesis'$ and all 
$\theta, \theta' \in \support(\Theta)$, we can reduce the problem to
\ECD with hypotheses $\set{\ferr(h, \theta)
  : h \in \H, \theta \in \support(\Theta) }$ and equivalence classes 
$\H_{i} := \set{\ferr(h_i, \theta) : \theta \in \support(\Theta) }$
for all $i$.
Then \thmref{thm:equiv-class-determination} immediately yields that 
the approximation factor of \algoecd is at most $2  \ln \paren{1/
  \min_{h,\theta} \prior(h, \theta)}  +1$, where the 
minimum is taken over all $(h,\theta)$ in the support of $\prior$.
In the unit cost case, running \algoecd with a modified prior \`{a} la
Kosaraju~et~al.~\cite{kosaraju99} allows us to obtain an
$\cO(\log |\hypotheses| + \log|\support(\Theta)|)$ approximation factor.
Note this model allows us to incorporate noise with complex
correlations.

\looseness -1 
However, a major challenge when dealing with noisy observations is
that it is not always possible to distinguish distinct hypotheses.
Even after we have run all tests, there will generally still be uncertainty
about the true hypothesis, i.e., the posterior distribution
$P(\hvar\mid\rlzvec)$ obtained using Bayes' rule may still
assign non-zero probability to more than one hypothesis. 
If so,  uniquely determining the true hypothesis is not possible.  Instead, we imagine that there is a set $\decset$ of
possible \emph{decisions} we may make after (adaptively) selecting a
set of tests to perform and we must choose one 
(e.g., we must decide how to treat the medical patient, which scientific
theory to adopt, or which classifier to use, given our
observations).  Thus our goal is to gather data to make effective decisions \cite{howard66voi}.
Formally, for any decision $\dec \in \decset$  we take, and each 
realized hypothesis $\hypothesis$,
we incur some loss $\loss(\dec,\hypothesis)$. Decision theory
recommends, after observing $\prlzvec{\cA}$, to choose the decision $d^{*}$ that minimizes the \emph{risk}, i.e., the
expected loss, namely $d^{*}\in\argmin_{d}
\expctover{\hvar}{\loss(\dec,\hvar)\mid\prlzvec{\cA}}$.

A natural goal in Bayesian active learning is thus to adaptively pick
observations, until we are guaranteed to make the same decision (and
thus incur the same expected loss) that we would have made had we
run \emph{all} tests. Thus, we can reduce the noisy
Bayesian active learning problem to the \ECDs problem by defining the
equivalence classes over all test outcomes that lead to the same  minimum
risk decision. Hence, for each decision $\dec\in\decset$, we define
\begin{equation}
  \label{eq:noisy-equiv-classes-def}
  \H_{\dec}:=\{\prlzvec{\groundset}: \dec=\argmin_{\dec'}
\expctover{\hvar}{\loss(\dec',\hvar)\mid \prlzvec{\groundset}}\}. \vspace{\gnat}
\end{equation}
If multiple decisions minimize the risk for a particular $\prlzvec{\groundset}$, we break ties arbitrarily.
Identifying the best decision $d \in \decset$ then amounts to
identifying which equivalence class $\H_{\dec}$ contains the realized
vector of outcomes, which is an instance of ECD.

\looseness -1 One common approach to this problem is to myopically pick tests 
maximizing the decision-theoretic \emph{value of information}
(VoI):
$\diff{VoI}{\prlzvec{\cA}}{\test} := \min_{\dec}
\expctover{\hvar}{\loss(\dec,\hvar)\mid \prlzvec{\cA}} - 
\expctover{\testval_{\test} \sim \testvar_{\test}\mid
  \prlzvec{\cA}}
{\min_{\dec} \expctover{\hvar}{\loss(\dec,\hvar)\mid \prlzvec{\cA}, \testval_{\test}}}$.
The VoI of a test $\test$ is the expected reduction in the
expected loss of the best decision due to the observation of 
$\testval_{\test}$.  
However, we can show there are
instances in which such a policy pays $\Omega(n/\log(n))$ times the
optimal cost, even under a uniform prior on $(h, \theta)$ and with 
$|\support(\Theta)| = 2$.  
\ifthenelse{\boolean{istechrpt}}{
Refer to Appendix~\ref{sec:infogain-stinks} for details.
}{ 
See the long version of this paper~\cite{dtree-arxiv-version} for details.
}
In contrast, on such instances \algoecd obtains an $\cO(\log n)$
approximation.  
More generally, we have the following result for \algoecd as an
immediate consequence of \thmref{thm:equiv-class-determination}.
\ignore{
\begin {figure} [t]
	\begin {center}
		\includegraphics [height=1.5in]{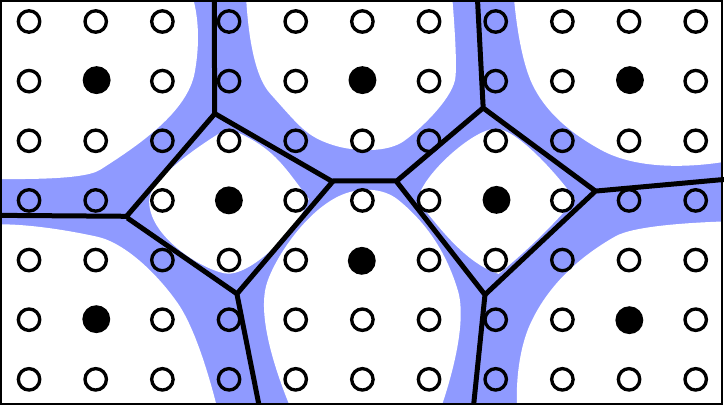} 
	\end {center}
	\caption {A schematic representation of the modified hypothesis space.
          Here, each circle represents a vector of outcomes for the
          complete set of queries $\data$.  Solid circles represent
          the (noiseless) signature of hypotheses; the other circles
          represent noisy versions of same.  The black lines indicate
          the cell boundaries, where a cell is a maximal set of points
        with the same MAP estimate.  The shaded region denote areas of
      severe error.  In \thmref{thm:noisy-case}, we condition on the
      realization lying outside the shaded region.
\daniel{Define ``modified hypothesis space'' (i.e., assuming MAP
  estimate after all queries made is correct), ``signature of a hypothesis'', and ``cells''.}
}
	\label {fig:voronoi}
\end {figure}
} %

\begin{theorem} \label{thm:noisy-case}
Fix hypotheses $\hypotheses$, tests $\tests$ with costs
$\cost(\test)$ and outcomes in $\outcomes$, decision set $\decset$, and
loss function $\ell$.
Fix a prior $P(\hvar, \Theta)$ and a function $\ferr:\hypotheses \times \support(\Theta)  \to
\outcomes^\ntests$ which define the probabilistic noise model.
Let $\cavgdec{\policy}$ denote the expected cost of $\policy$ incurs to
find the best decision, i.e., to identify which equivalence class
$\H_{\dec}$ the outcome vector $\btestsval_{\tests}$ belongs to.
Let $\policy^*$ denote the policy minimizing $\cavgdec{\cdot}$, and let 
$\policy_{EC}$ denote the adaptive policy implemented by  \algoecd. 
Then it holds that
$$\cavgdec{\policy_{EC}} \le \paren{2
  \ln \paren{\frac{1}{p'_{\min}}} +1 }
\cavgdec{\policy^*}, $$
where $p'_{\min} := \min_{\hypothesis \in \hypotheses}
\set{\prior(h,\theta) \ :\ \prior(h,\theta) > 0}$.
\end{theorem}

If all tests have unit cost, by using
a modified prior~\cite{kosaraju99} the approximation factor can be improved to 
$\cO\paren{\log |\hypotheses| +  \log |\support(\Theta)|}$ 
as in the case of \thmref{thm:equiv-class-determination}.

\ignore{
\begin{theorem} \label{thm:noisy-case-old}
Fix any probabilistic model $P(\hvar,\testvec{\tests})$ over
hypotheses $\hvar\in\hypotheses$ and tests
$\testvar_{\test}\in\outcomes$ with costs $\cost(\test)$, decision set
$\decset$ and loss function $\ell$.  Further 
let $\delta > 0$ be such that 
$P(\prlzvec{\groundset} \mid \hypothesis)\ge \delta$
for all $(h, \prlzvec{\groundset})$ in the support of
$P$.  
Let $\policy^*$ denote the policy minimizing $\cavgdec{\cdot}$, and let 
$\policy_{EC}$ denote the adaptive policy implemented by  \algoecd. %
Then it holds that
$$\cavgdec{\policy_{EC}} \le \paren{2
  \ln \paren{\frac{1}{p_{\min}}}  + 2 \ln \paren{\frac{1}{\delta}} +1 }
\cavgdec{\policy^*}, $$
where $p_{\min} := \min_{\hypothesis \in \hypotheses} \prior(h)$.
If all tests have unit cost, the approximation factor can be improved to 
$\cO(\log |\hypotheses|/\delta)$.
\end{theorem}
} %

\paragraph{The \algoname algorithm.}
For some noise models, $\Theta$ may have exponentially--large
support.  In this case reducing Bayesian active learning with noise to
\ECD results in instances with exponentially-large equivalence
classes.  This makes running \algoecd on them challenging, since 
explicitly keeping track of the equivalence classes is impractical.
To overcome this challenge,
we develop \algoname, a particularly efficient algorithm which
approximates 
\algoecd.

\daniel{Some issues that could be improved: $i$ is overloaded as a
  test, or a equiv class index, we're using $h_i$ to refer to classes $\H_d$.}

For clarity, 
we only consider the $0\!-\!1$ loss, i.e.,
our goal is to find the most likely hypothesis (MAP estimate) 
given all the data $\prlzvec{\data}$, namely
$\htrue(\prlzvec{\data}) :=\argmax_{\hypothesis}P(\hypothesis\mid
\prlzvec{\data})$. Recall definition (\ref{eq:noisy-equiv-classes-def}), and consider the weight of edges between distinct
equivalence classes $\H_{i}$ and $\H_{j}$:
\vspace{\gnat}
$$w(\H_{i}\times\H_{j})= \hspace{-2em}\sum_{\rlzvec
  \in\H_{i},\rlzvec'\in\H_{j}} \hspace{-1.9em}
\prior(\rlzvec)\prior(\rlzvec')\!=\!\Bigl(\sum_{\rlzvec\in\H_{i}}
\hspace{-0.5em} \prior(\rlzvec)\Bigr)\Bigl(\sum_{\rlzvec'\in\H_{j}} \hspace{-0.5em} \prior(\rlzvec')\Bigr)=\prior(\testvec{\data}\in\H_{i})\prior(\testvec{\data}\in\H_{j}). \vspace{\gnat}$$
In general,  $P(\testvec{\data}\in\H_{i})$ can be estimated to
arbitrary accuracy using a rejection sampling approach with bounded
sample complexity. We defer details to the full version of the paper. Here, we focus on the case where, upon observing all tests, the hypothesis is uniquely determined, i.e., $P(\hvar\mid\rlzvec)$ is deterministic for all $\rlzvec$ in the support of $P$.
In this case, it holds that $P(\testvec{\data}\in\H_{i})=P(\hvar=\hypothesis_{i})$. 
Thus, the total weight is 
$$\sum_{i\neq j}w(\H_i\times\H_{j})=\Bigl(\sum_{i}P(\hypothesis_{i})\Bigr)^{2}-\sum_{i}P(\hypothesis_{i})^{2}=1-\sum_{i}P(\hypothesis_{i})^{2}.$$
This insight motivates us to use the objective function
$$\diff{Eff}{\prlzvec{\cA}}{\test} :=\Bigl[\sum_{\outcome}P(\testvar_{\test}=\outcome\mid\prlzvec{\cA})
\Big(\sum_{i}P(\hypothesis_{i}\mid \prlzvec{\cA},\testvar_{\test}=\outcome)^{2}\Big)\Bigr]-\sum_{i}P(\hypothesis_{i}\mid \prlzvec{\cA})^{2},$$
which is the expected reduction in weight from the prior to the posterior distribution. Note that the weight of a distribution $1-\sum_{i}P(\hypothesis_{i})^{2}$ is a monotonically increasing function of the R\'{e}nyi entropy (of order 2), which is $-\frac{1}{2}\log\sum_{i}P(\hypothesis_{i})^{2}$. Thus the objective $\badtassdelta$ can be interpreted as a (non-standard) information gain in terms of the (exponentiated) R\'{e}nyi entropy. 
In our experiments, we show that this criterion performs well in comparison to existing experimental design criteria, including the classical Shannon information gain. Computing $\diff{Eff}{\prlzvec{\cA}}{\test}$ requires us to perform one inference task for each outcome $\outcome$ of $\testvar_{\test}$, and $\mathcal{O}(\nhypo)$ computations to calculate the weight for each outcome. 
We call the  algorithm that greedily optimizes $\badtassdelta$ the
\algoname algorithm (since it uses an
Efficient Edge Cutting approXimate objective), and present pseudocode in
\algref{alg:greedy}. 

\begin{algorithm}
 \KwIn{Set of hypotheses $\hypotheses$; Set of tests $\data$; prior distribution $P$; function $f$.}
  \Begin{
   $\cA\leftarrow\emptyset$\;
   \While{$\exists \hypothesis\neq\hypothesis': P(\hypothesis\mid\prlzvec{\cA})>0\text{ and } P(\hypothesis'\mid\prlzvec{\cA})>0$}{
       \ForEach{$\test\in\data\setminus \cA$}{
$\diff{Eff}{\prlzvec{\cA}}{\test} := \! \Bigl[\sum_{\outcome}\! P(\testvar_{\test}=\outcome\mid\prlzvec{\cA})
\paren{\sum_{i}\! P(\hypothesis_{i}\mid \prlzvec{\cA},\testvar_{\test}=\outcome)^{2}}\Bigr]-\sum_{i}P(\hypothesis_{i}\mid \prlzvec{\cA})^{2}$\;    
       }
       Select $\test^{*}\in\argmax_{\test} \diff{Eff}{\prlzvec{\cA}}{\test}/\cost(\test)$;  Set $\cA\leftarrow \cA\cup\{\test^{*}\}$ 
and observe outcome $\outcome_{\test^{*}}$\;%
   }
 } \label{alg:greedy} \caption{The \algoname algorithm using the 
Efficient Edge Cutting approXimate objective.}
\end{algorithm}

\vspace{\gnat}\section{Experiments} \label{sec:experiments}\vspace{\gnat}
\looseness -1 
Several economic theories make claims to explain how people make
decisions when the payoffs are uncertain. Here we use human subject
experiments to compare four key theories proposed in literature. 
The uncertainty of the payoff in a given situation is represented by a
\emph{lottery} $L$, which is simply a random variable with a range of
\emph{payoffs} $\payoffs := \set{\payoff_1, \ldots, \payoff_{k}}$.
For our purposes, a payoff is an integer denoting how many dollars you
receive (or lose, if the payoff is negative).
Fix lottery $L$, and let $p_i := \prob{L = \payoff_i}$.
The four theories posit distinct utility functions, with agents
preferring larger utility lotteries.  Three of the theories have
associated parameters.
The \emph{Expected Value theory}~\cite{vonneumann47} posits simply 
$U_{EV}(L) = \expct{L}$, and has no parameters.
\emph{Prospect theory}~\cite{kahneman79} posits 
$U_{PT}(L) = \sum_{i} f(\payoff_{i}) w(p_{i})$ for nonlinear functions
$f(\payoff_{i}) = \payoff_{i}^{\rho}$, if $\payoff_{i} \geq 0$ and
$f(\payoff_{i}) = -\lambda (-\payoff_{i})^{\rho}$, if $\payoff_{i} < 0$, and
$w(p_{i}) = e^{-(\log(1/p_{i}))^{\alpha}}$~\cite{prelec98}. The
parameters $\Theta_{PT}=\{\rho, \lambda, \alpha\}$ represent risk
aversion, loss aversion and probability weighing factor
respectively.  For portfolio optimization problems,
financial economists have used value functions that give 
weights to different moments of the lottery~\cite{hanoch70}: $U_{MVS}(L) = w_{\mu}\mu -
w_{\sigma}\sigma + w_{\nu}\nu$, where $\Theta_{MVS} = \{w_{\mu},
w_{\sigma}, w_{\nu}\}$ are the weights for the mean, standard
deviation and standardized skewness of the lottery respectively. 
In
\emph{Constant Relative Risk Aversion theory}~\cite{pratt64}, there is
a parameter $\Theta_{CRRA} = a$ representing the level of risk aversion, and the
utility posited is 
$U_{CRRA}(L)=\sum_{i} p_{i}\payoff_{i}^{1-a}/(1-a)$ if $a \neq 1$, and $U_{CRRA}(L)
=\sum_{i} p_{i} \log(\payoff_{i})$, if $a=1$. 

The goal is to adaptively select a sequence of tests to present to a
human subject in order to distinguish which of the four theories 
best explains the subject's responses. 
Here a test $\test$ is a pair of lotteries, $(L_1^{\test}, L_2^{\test})$.
Based on the theory that represents behaviour, one of the
lotteries would be preferred to the other, denoted by a binary
response $\outcome_{\test} \in \{1, 2\}$.  The possible payoffs were fixed to $\payoffs
= \{-10, 0, 10\}$ (in dollars), and the distribution $(p_{1}, p_{2}, p_{3})$ over the
payoffs was varied, where $p_{i} \in
\set{0.01, 0.99} \cup \set{0.1, 0.2,\ldots, 0.9}$.  
By considering all non-identical pairs of such lotteries, we obtained
the set of possible tests.

\daniel{Andreas, please verify my def of uncertainty sampling.}

\looseness -1 We compare six algorithms:
\algoname, greedily maximizing Information Gain
(IG), Value of Information (VOI), Uncertainty Sampling\footnote{Uncertainty sampling greedily selects the test whose outcome distribution has maximum Shannon entropy.} (US), Generalized Binary Search (GBS), and
tests selected at Random. We evaluated the ability of the algorithms to 
recover the true model based on simulated responses. We chose
parameter values for the theories 
such that they made distinct predictions and were consistent with the
values proposed in literature~\cite{kahneman79}. We drew $1000$
samples of the true model and fixed the parameters of the model to
some canonical values, $\Theta_{PT}=\{0.9,2.2,0.9\},
\Theta_{MVS}=\{0.8,0.25,0.25\}, \Theta_{CRRA}=1$. Responses were
generated using a softmax function, with the probability of response $\testval_{\test}=1$
given by $P(\testval_{\test}=1) = 1/(1+e^{U(L_{2}^{\test})-U(L_{1}^{\test})})$. \figref{fig:experiments}(a)
shows the performance of the $6$ methods, in terms of the accuracy of
recovering the true model with the number of tests. We find that US, GBS 
and VOI perform significantly worse than Random in the presence of
noise. \algoname outperforms InfoGain significantly, which outperforms
Random.  

We also considered uncertainty in the values of the parameters, by
setting $\rho$ from 0.85-0.95, $\lambda$ from 2.1-2.3, $\alpha$ from
0.9-1; $w_{\mu}$ from 0.8-1.0, $w_{\sigma}$ from 0.2-0.3, $w_{\nu}$
from 0.2-0.3; and $a$ from 0.9-1.0, all with 3 values per
parameter. We generated 500 random samples by first randomly sampling
a model and then randomly sampling parameter values. \algoname and
InfoGain outperformed Random significantly,
\figref{fig:experiments}(b), although InfoGain did marginally better
among the two. The increased parameter range potentially poses model identifiability 
issues, and violates some of the assumptions behind \algoname, decreasing its
performance to the level of InfoGain.

\begin{figure}[t]
\centering \subfigure[\emph{Fixed parameters}]{
\includegraphics[width=0.31\textwidth]{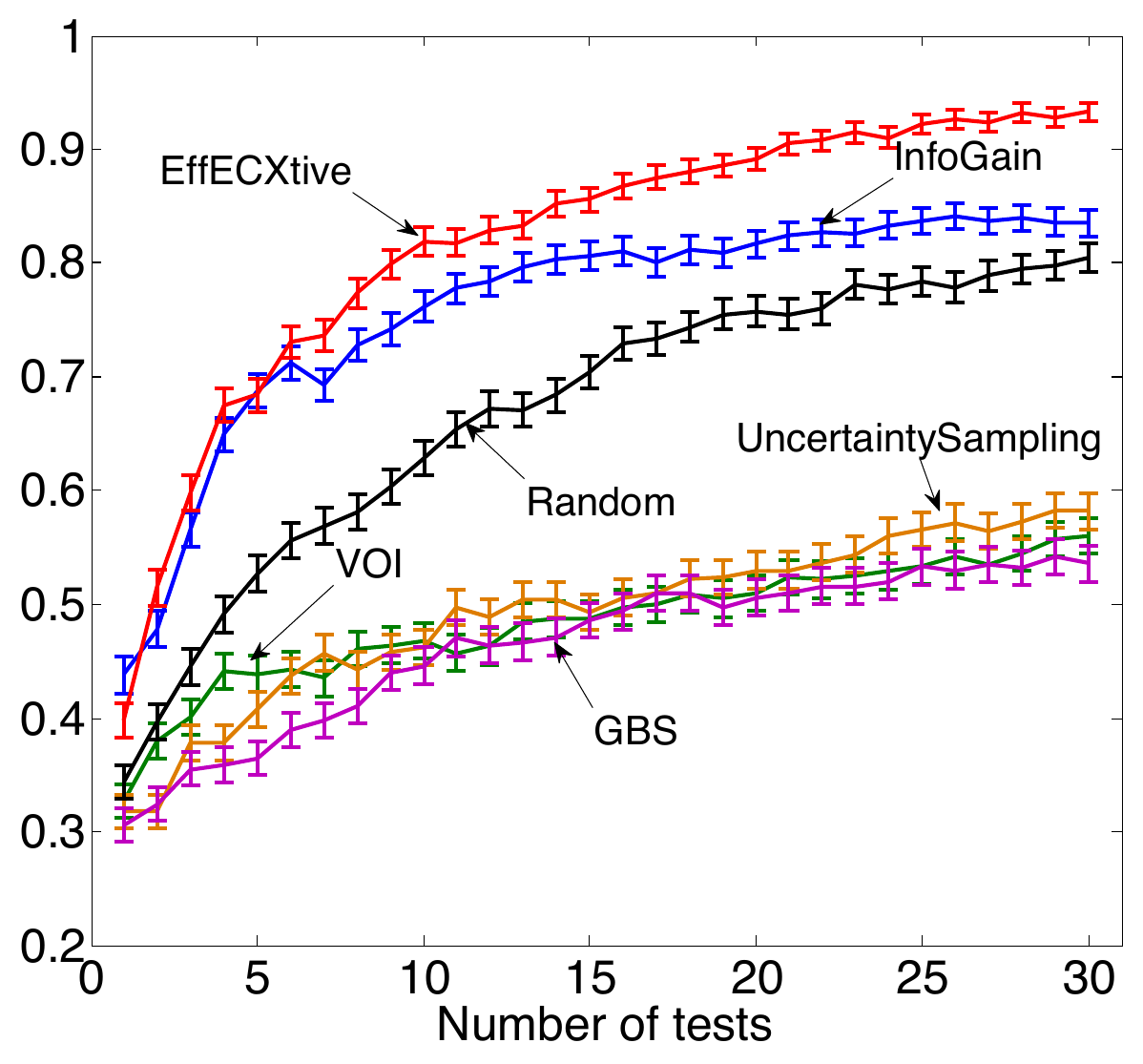}
 \label{fig:experiment-fixed-params}
 }
\subfigure[\emph{With parameter uncertainty}]{
\includegraphics[width=0.31\textwidth]{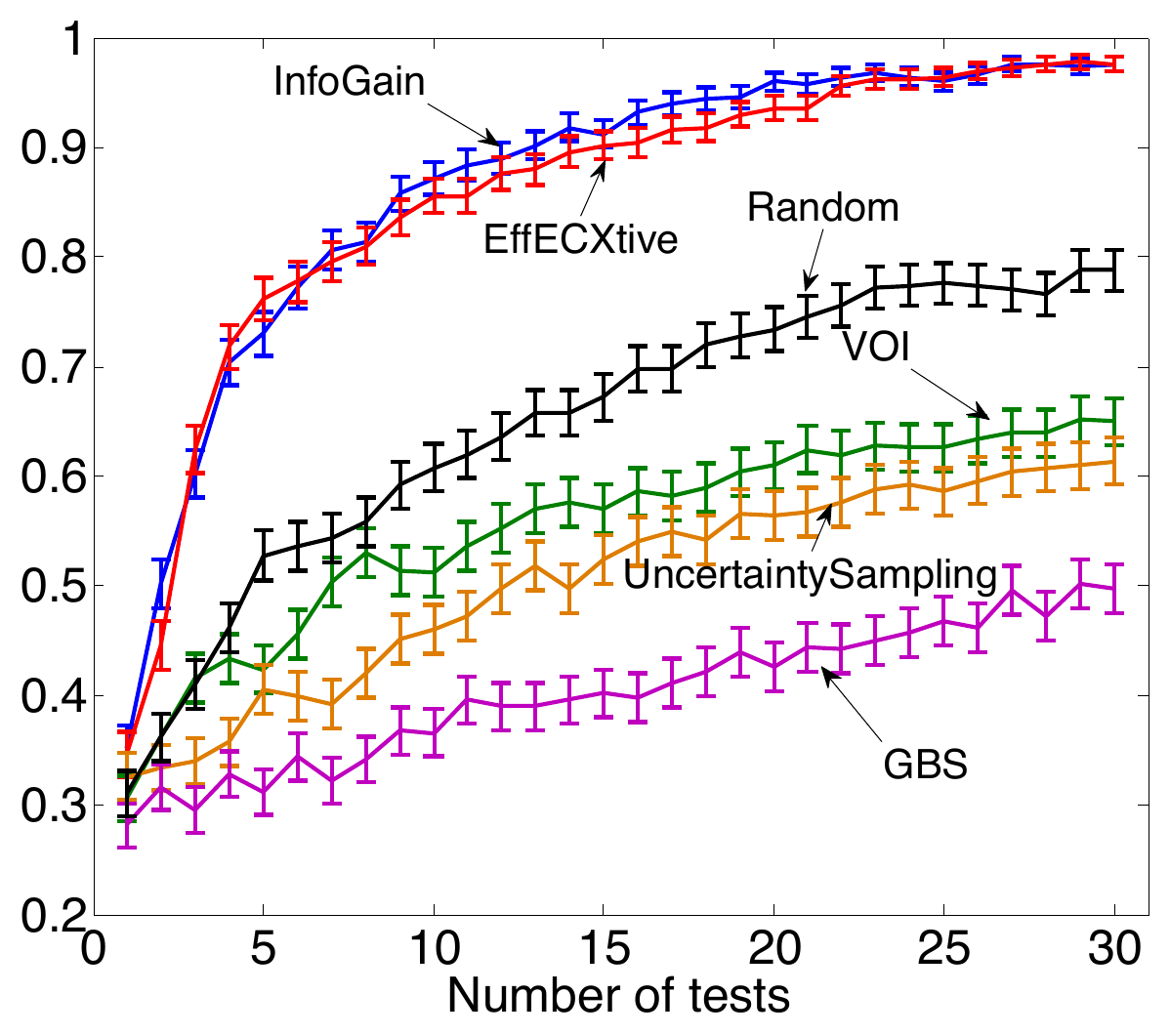}
 \label{fig:experiment-uncertain-params}
 }
\subfigure[\emph{Human subject data}]{
\includegraphics[width=0.32\textwidth]{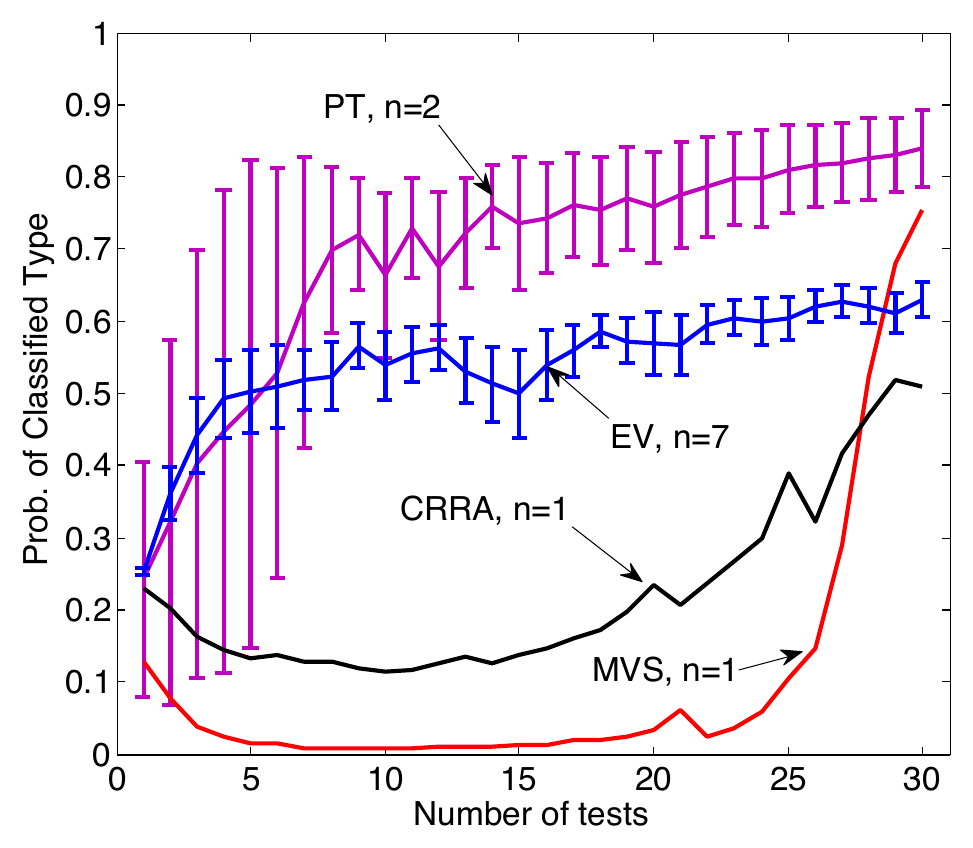}
 \label{fig:experiment-human-subjects}
 }
 \vspace{-3mm}
 \caption{\footnotesize (a) Accuracy of identifying the true model with fixed parameters,
(b) Accuracy using a grid of parameters, incorporating uncertainty in their values,
(c) Experimental results: 11 subjects were classified into the
theories that described their behavior best. We plot probability of
classified type.\label{fig:experiments}
 }\vspace{-3mm}
\end{figure}

After obtaining informed consent according to a protocol approved by
the Institutional Review Board of Caltech, we tested 11 human subjects
to determine which model fit their behaviour best. Laboratory
experiments have been used previously to distinguish economic
theories, \cite{camerer89}, and here we used a real-time, dynamically
optimized experiment that required fewer tests. Subjects were
presented 30 tests using \algoname.  To incentivise the subjects, one
of these tests was picked at random, and subjects received payment
based the outcome of their chosen lottery. The behavior of most
subjects (7 out of 10) was best described by EV. This is not
unexpected given the high quantitative abilities of the
subjects. We also found heterogeneity in classification: One subject
got classified as MVS, as identified by violations of stochastic
dominance in the last few choices. 2 subjects were best described by
prospect theory since they exhibited a high degree of loss aversion
and risk aversion. One subject was also classified as a CRRA-type
(log-utility maximizer). Figure~\ref{fig:experiments}(c) shows the
probability of the classified model with number of tests. Although we
need a larger 
sample to 
make significant claims of
the validity of different economic theories, our preliminary results
indicate that subject types can be identified and there is
heterogeneity in the population.  They also serve as 
an example of the benefits of using 
real-time dynamic experimental design 
to collect data on human behavior.

\vspace{\gnat}\section{Conclusions} \label{sec:conclusions}\vspace{\gnat}
\looseness -1 In this paper, we considered the problem of adaptively
selecting which noisy tests to perform in order to identify an unknown
hypothesis sampled from a known prior distribution. We studied the
Equivalence Class Determination problem as a means to reduce the case
of noisy observations to the classic, noiseless case. 
We introduced \algoecd, an adaptive greedy algorithm that is
guaranteed to choose the same hypothesis as if it had observed the
outcome of all tests, and incurs near-minimal expected cost among all
policies with this guarantee.  This is in contrast to popular
heuristics that are greedy w.r.t. version space mass reduction,
information gain or value of information, all of which we show can be 
very far from optimal.
\algoecd works by greedily optimizing an objective tailored to
differentiate between sets of observations that lead to different
decisions. Our bounds rely on the fact that this objective function
is adaptive submodular. We also develop \algoname, a practical 
algorithm based on \algoecd, that can be applied to arbitrary
probabilistic models in which efficient exact inference is
possible. We apply \algoname to a Bayesian experimental design
problem, and our results indicate its effectiveness in comparison to
existing algorithms. We believe that our results provide an
interesting direction towards providing a theoretical foundation for
practical active learning and experimental design problems.

\ifthenelse{\boolean{anonymous}}{ }{

{\small
\paragraph{Acknowledgments.}
This research was partially supported by ONR grant N00014-09-1-1044,
NSF grant CNS-0932392, NSF grant IIS-0953413, a 
gift by Microsoft Corporation, an Okawa Foundation Research Grant, and
by the Caltech Center for the Mathematics of Information.
The authors also thank Yuxin Chen and Amin Karbasi for pointing out an
error in the original proof of
Lemma~\ref{lem:classification-objective-submod} and simultaneously
providing a fix.
}

}
\ifthenelse{\boolean{istechrpt}}
{ }
{\newpage}
\bibliographystyle{plain} %
{\small
\bibliography{dtree}
}

\ifthenelse{\boolean{istechrpt}}{ 
\appendix
\section{Additional Proofs} \label{sec:proofs}

\newcommand{\nvs}{\ensuremath{n_{\cV}}}

\begin{lemma} \label{lem:classification-objective-monotone}
The objective function $f$ of~\eqnref{eq:class1} is strongly
adaptive monotone.  
\end{lemma}

\begin{proof}
 We must show that for all $\prlzvec{\cA}$, $\test \notin \cA$ and
 possible answer $\outcome$ for test $\test$ that 
 \begin{equation}
   \label{eq:class2}
   \expctover{\hvar}{f(\cA, \hvar) \ \mid \ \prlzvec{\cA}} \le 
\expctover{\hvar}{f(\cA \cup \set{\test}, \hvar) \ \mid \ 
  \prlzvec{\cA}, \testvar_{\test} = \outcome} %
 \end{equation}
Towards this end, it is useful to notice that for all $\test \in \tests$ the function $\hypothesis
\mapsto \cut{\test}{\hypothesis}$ depends only on $\testvar_{\test}$.
Hence for any
$\prlzvec{\cA}$, the function $\hypothesis \mapsto f(\cA, \hypothesis)$ is constant
over realizations $\prlzvec{\groundset} \succ \prlzvec{\cA}$, so we
can define a function $g(\prlzvec{\cA})$ such that 
$g(\prlzvec{\cA}) = \expctover{\hvar}{f(\cA, \hvar) \mid \prlzvec{\cA}}$
by 
$g(\prlzvec{\cA}) := \w\paren{\bigcup_{\test \in \cA} \cut{\test}{\outcome_{\test}} }$
where $\prlzvec{\cA} = (\outcome_{\test})_{\test \in \cA}$ and
$\cut{\test}{\outcome}$ is the set of edges cut by $\test$ if
$\testvar_{\test} = \outcome$.
Note that for all $\prlzvec{\cA} \prec \prlzvec{\cB}$ we have
$g(\prlzvec{\cA}) \le g(\prlzvec{\cB})$, since the edge weights are
nonnegative.  Setting $\cB = \cA \cup \set{\test}$
yields~\eqnref{eq:class2} and hence implies strong adaptive
monotonicity.
\end{proof}

\begin{lemma} \label{lem:classification-objective-submod}
The objective function $f$ of~\eqnref{eq:class1} is adaptive
submodular for any prior with rational values.
\end{lemma}

\begin{proof}
We first prove the result assuming a uniform prior $\prior(\cdot)$,
and then show how to reduce the general prior case to the uniform
prior case.  Hence all edges have weight $1/n^2$, where there are $n$
hypotheses.  For convenience, we also rescale our units of reward so that 
all edges have unit weight.  (Note that $f$ is adaptive submodular iff
$c f$ is for any constant $c > 0$.) 
To prove adaptive submodularity, we must show that for all 
$\prlzvec{\cA} \prec \prlzvec{\cB}$ and $\test \in \tests$, we have 
$\diff{}{\prlzvec{\cB}}{\test} \le \diff{}{\prlzvec{\cA}}{\test}$.
Fix $\test$ and $\prlzvec{\cA}$, and let $\vs(\prlzvec{\cA}) :=  \set{h : \prior(h \mid \prlzvec{\cA}) > 0}$ denote the version space, if
$\prlzvec{\cA}$ encodes the observed outcomes.  
Let $\nvs := |\vs(\prlzvec{\cA})|$ be the number of 
hypotheses in the version space.
Likewise, let 
$n_{i,a}(\prlzvec{\cA}) := |\set{h : h \in \vs(\prlzvec{\cA}, \testvar_{\test} = a) \cap
  \H_i}|$, and let $n_a(\prlzvec{\cA}) := \sum_{i=1}^\nclass n_{i,a}(\prlzvec{\cA})$.
We define a function $\theta$ of the quantities $\set{n_{i,a} : 1 \le
  i \le \nclass, a \in \outcomes}$
 such that 
$\diff{}{\prlzvec{\cA}}{\test} = \theta(\vectorize{n}(\prlzvec{\cA}))$, where 
$\vectorize{n}(\prlzvec{\cA})$ is the vector consisting of $n_{i,a}(\prlzvec{\cA})$ for all
$i$ and $a$.  For brevity, we suppress the dependence
of $\prlzvec{\cA}$ where it is unambiguous.

It will be convenient to define $e_{a}$ to be the number of
edges cut by $\test$ such that at $\test$ both hypotheses agree with each other but disagree with the realized
hypothesis $\htrue$, conditioning on $\testvar_{\test} = a$.
Written as a function of $\vectorize{n}$, we have  
$e_a(\vectorize{n}) := \sum_{i < j} \sum_{b \neq a} n_{i,b} \cdot n_{j, b}$.

As we will explain below, given this expression for $e_a$ we can define $\theta$ as
\begin{equation}
  \label{eq:class3}
  \theta(\vectorize{n}) := \sum_{i < j} \sum_{a \neq b}
n_{i,a} \cdot n_{j, b} + \sum_{a} e_{a} \paren{\frac{n_{a}}{\nvs}}
\end{equation}
Here, $i$ and $j$ range over all class indices, and $a$ and $b$
range over all possible outcomes of test $\test$.
The first term on the right-hand side counts the number of edges that
will be cut by selecting test $\test$ no matter what the outcome of $\test$
is.  Such edges consist of hypotheses that disagree with each other at
$\test$ and, as with all edges, lie in
different classes.
The second term counts the expected number of edges cut by $\test$ 
consisting of hypotheses that agree with each other at $\test$.
Such edges will be cut by $\test$ iff they disagree with $\rlz^*$ at $\test$.
The edges $\set{h, h'}$ with $h, h' \in \vs(\prlzvec{\cA})$ and
$\prior(\testvar_{\test} = b \mid h) = \prior(\testvar_{\test} = b
\mid h') = 1$ for some $b \neq a$
 (of which there are a total of $e_{a}$) will be cut by $\test$ iff $\testvar_{\test} =
a$.  Since we assume a uniform prior, $\prob{\testvar_{\test} = a \ \mid \
  \prlzvec{\cA}} = n_a / \nvs$ for any partial realization $\prlzvec{\cA}$ with $\test \notin \cA$, hence the expected contribution of these edges to $\diff{}{\prlzvec{\cA}}{\test}$
is $\sum_a e_a\paren{n_a/ \nvs}$, from whence we get the second term.

Now fix $\prlzvec{\cB} \succ \prlzvec{\cA}$.
Our strategy for proving  $\diff{}{\prlzvec{\cB}}{\test} \le \diff{}{\prlzvec{\cA}}{\test}$
is as follows.  As more observations are made, the version space can
only shrink, i.e. $\vs(\prlzvec{\cB}) \subseteq \vs(\prlzvec{\cA})$.  This means that
for all $i$ and $a$, $n_{i,a}$ is nonincreasing, i.e.,
$n_{i,a}(\prlzvec{\cB}) \le n_{i,a}(\prlzvec{\cA})$.
Hence we consider a parameterized path $p(\tau)$ in $\reals^{\nclass \nout}$
from $p(0) := \vectorize{n}(\prlzvec{\cB})$ to $p(1) := \vectorize{n}(\prlzvec{\cA})$.
Then by integrating along the path we obtain 
\begin{equation}
  \label{eq:class4}
  \diff{}{\prlzvec{\cA}}{\test} -\diff{}{\prlzvec{\cB}}{\test} = \int_{\tau =
    0}^{1} \paren{\frac{d (\theta \! \circ\! p)}{d\tau}} d\tau \mbox{.}
\end{equation}
We require that 
$p$ is nondecreasing in each coordinate as a function of $\tau$; in
other words, that
$\partial n_{i,a} / \partial \tau \ge 0$ for all classes $i$, outcomes $a$, and $\tau \in [0,1]$.
There exists a path meeting this requirement, since 
$\vectorize{n}(\prlzvec{\cB}) \le \vectorize{n}(\prlzvec{\cA})$.
Hence we can prove the integral is nonnegative by applying the
chain rule for the derivative to obtain
$$\frac{d (\theta \! \circ \! p)}{d\tau} \ = \ \ \sum_{i,a} \frac{\partial \theta}{\partial
  n_{i,a}}  \frac{\partial n_{i,a}}{\partial \tau}  $$
and then proving that  
$\partial \theta / \partial n_{i,a} \ge 0$
for all $i$ and $a$.

Fix a class index $k$ and an outcome $c$ and consider $\partial \theta / \partial n_{k,c}$.
As we prove in Lemma~\ref{lem:theta-derivative}, elementary calculus tells us that 
\begin{equation}
  \label{eq:class5}
\frac{\partial \theta}{\partial n_{k,c}} \ = \ 
 \frac{e_c}{\nvs} \ \ + \ \
\sum_{i \neq k, \ a \neq c} \frac{n_a n_{i,c}}{\nvs} \ \ + \ \ \sum_{i \neq k, \ a \neq c} n_{i,a} \ \ 
 - \ \ \sum_{b} \frac{e_bn_b }{\nvs^2}. 
\end{equation}
Multiplying~\eqnref{eq:class5} by $\nvs$, we see that this quantity is nonnegative iff 
\begin{equation}
  \label{eq:class6}
\sum_{b} \frac{e_bn_b }{\nvs}  \ \ \le \ \  e_c \ + \ \sum_{i \neq k,
  \ a \neq c} n_a n_{i,c} \ + \ 
\nvs \sum_{i \neq k, \ a \neq c}
n_{i,a} \
\end{equation}

We prove~\eqnref{eq:class6} via the following sequence of equations,
which are explained below.
\begin{eqnarray}
\label{eq:s1} 
\sum_{b} \frac{e_bn_b }{\nvs} & = & \frac{e_cn_c }{\nvs} + \sum_{b
  \neq c} \frac{e_bn_b }{\nvs} \\ \label{eq:s2}
 & \le & e_c + \sum_{b  \neq c} \frac{e_bn_b }{\nvs} \\   \label{eq:s3}
 & = & e_c + \sum_{b  \neq c} \frac{n_b }{\nvs} \paren{ \sum_{i < j} \sum_{a \neq b}
n_{i,a} \cdot n_{j, a} } \\  \label{eq:s4}
 & = & e_c + \sum_{b  \neq c} \frac{n_b }{\nvs} \paren{ \sum_{i < j}
   n_{i,c} n_{j, c} +  \sum_{i < j}\sum_{a \neq b, a \neq c} 
n_{i,a} \cdot n_{j, a} } \\ \label{eq:s5}
 & \le & e_c + \sum_{b  \neq c} \frac{n_b }{\nvs} \paren{ \paren{\sum_{i \neq
   k} n_{i,c} } n_{c} }  +  \sum_{b  \neq c} \frac{n_b
}{\nvs} \paren{\nvs \sum_{ i \neq k, a \neq c} n_{i,a} } \\ \label{eq:s6}
 & = & e_c + \paren{\sum_{b  \neq c} n_b\frac{ n_{c}}{\nvs}} \paren{\sum_{i \neq
   k} n_{i,c} }   +  \paren{\sum_{b \neq c} \frac{n_b
}{\nvs}} \paren{\nvs \sum_{ i \neq k, a \neq c} n_{i,a}
} \\ \label{eq:s7} 
& \le & e_c + \paren{\sum_{b  \neq c} n_b} \paren{\sum_{i \neq
   k} n_{i,c} }   +  \nvs \sum_{ i \neq k, a \neq c} n_{i,a}  \\ \label{eq:s8} 
& = & e_c + \sum_{a  \neq c, i \neq k} n_a n_{i,c}    +  \nvs \sum_{ i \neq k, a \neq c} n_{i,a}  
 \end{eqnarray}

The first four equations are straightforward, separating out terms for
outcome $c$ in \eqnref{eq:s1} and \eqnref{eq:s4}, using $n_c / \nvs
\le 1$ in \eqnref{eq:s2} and the definition of $e_b$ in \eqnref{eq:s3}.
For \eqnref{eq:s5} we first make a simple observation:
Let $\set{x_i}_{i \ge 0}$ be a finite sequence of non-negative real
numbers.  Then for any $k$, $\sum_{i < j} x_i x_j \le \paren{\sum_i x_i} \paren{\sum_{i \neq k} x_i}$.
We use this fact twice in \eqnref{eq:s5}.
First, we apply it with $x_i = n_{i,c}$, so that after noting
$n_c = \sum_i x_i$, we conclude $\sum_{i < j} n_{i,c} n_{j,c} \le \paren{\sum_{i \neq k} n_{i,c}}n_c$.
Second, we apply it with $x_i = n_{i,a}$ for each $a \notin \set{b,
  c}$, and conclude 
$$\sum_{i < j} \sum_{a \neq b, a \neq c}  n_{i,a} n_{j,a} \le \sum_{a
  \neq b, a \neq c} \paren{\sum_{i \neq k} n_{i,a}}n_a \le \nvs \sum_{a
  \neq b, a \neq c} \paren{\sum_{i \neq k} n_{i,a}}  \le \nvs \sum_{a \neq c} \paren{\sum_{i \neq k} n_{i,a}}.$$
For~\eqnref{eq:s6} and~\eqnref{eq:s8} we merely rearrange terms for
the reader's convenience.  In~\eqnref{eq:s7}
we simply use $n_c / \nvs \le 1$ and $\sum_{b \neq c} n_b / \nvs \le 1$.
From this we conclude that $\frac{\partial \theta}{\partial n_{k,c}}
\ge 0$ for all $k, c$ and $\tau$, and hence that 
$\diff{}{\prlzvec{\cB}}{\test} \le \diff{}{\prlzvec{\cA}}{\test}$, and
so $f$ is adaptive submodular under a uniform prior.

We next show how to reduce the general prior case to the uniform prior
case.  Fix any prior $\prior$ with rational probabilities,
i.e. $\prior(h) \in \rationals$ for all $h$.  Then there exists $d \in
\nats$ and function $k:\hypotheses \to \nats$ such that 
such that $\prior(h) = k(h)/d$.  Create a new instance containing $d$
hypotheses, where for each $h  \in \hypotheses$ there are $k(h)$ copies of $h$, denoted by $h^1, \ldots, h^{k(h)}$.
Each copy of $h$ induces the same conditional distribution of test
outcomes $\prior(\testvar_{1}, \ldots, \testvar_{\ntests} \mid h)$.
All copies of $h$ belong to the same class, and copies of $h$
and $h'$ belong to the same class iff $h$ and $h'$ do.
Finally, assign a uniform prior to this new instance.
Then the adaptive submodularity of $f$ on this new instance implies
the adaptive submodularity on the original instance, if the weight of 
edge $\set{h, h'}$ in the original instance is proportional to the
number of edges between the copies of $h$ and the copies of $h'$ in the new instance.
That is, it suffices to set $\w(\set{h, h'}) \propto k(h) \cdot
k(h')$, and our choice of weight function, $\w(\set{h, h'}) := \prior(h) \cdot
\prior(h')$, satisfies this condition.
\end{proof}

\begin{lemma}  \label{lem:theta-derivative}
The partial derivatives of $\theta$ are given by 
$$ \frac{\partial \theta}{\partial n_{k,c}} \ = \ 
 \frac{e_c}{\nvs} \ \ + \ \
\sum_{i \neq k, \ a \neq c} \frac{n_a n_{i,c}}{\nvs} \ \ + \ \ \sum_{i \neq k, \ a \neq c} n_{i,a} \ \ 
 - \ \ \sum_{b} \frac{e_bn_b }{\nvs^2} $$
\end{lemma}

\begin{proof}
Recall $$\theta(\vectorize{n}) := \sum_{i < j} \sum_{a \neq b}
n_{i,a} \cdot n_{j, b} + \sum_{a} \frac{e_{a} n_{a}}{\nvs}.$$
The partial derivative of the first term is relatively straightforward:
$$
\frac{\partial }{\partial n_{k,c}} \paren{ \sum_{i < j} \sum_{a \neq b}
n_{i,a} \cdot n_{j, b} } = \sum_{i \neq k, a \neq c} n_{i,a}.
$$
The partial derivative of the second term is:
\begin{eqnarray}
\frac{\partial}{\partial n_{k, c}} \paren{ \sum_{a} \frac{e_a n_a}{\nvs} } 
& = & \sum_{a} \frac{\partial}{\partial n_{k, c}} \paren{ \frac{e_a n_a}{\nvs} } \\
& = &
\frac{\partial}{\partial n_{k, c}} \paren{\frac{e_c n_c}{\nvs} } + 
 \sum_{a\neq c}  \frac{\partial}{\partial n_{k,
    c}} \paren{\frac{e_a n_a}{\nvs} } \\
& = &
 \frac{n_c}{\nvs} \cdot  
\underbrace{\frac{\partial e_c }{ \partial n_{k, c} }}_{=\ 0} 
\ + \ \frac{e_c}{\nvs} \cdot 
\underbrace{\frac{\partial n_c }{ \partial n_{k, c} } 
}_{= \ 1}
\ + \ e_c n_c \cdot \frac{\partial }{ \partial n_{k, c}
} \paren{\frac{1}{\nvs}}  \nonumber \\
& & + \ 
 \sum_{a\neq c} \left\{ %
\frac{n_a}{\nvs} \cdot \hspace{-3mm}
\underbrace{ \frac{\partial e_a }{ \partial n_{k, c} }  }_{= \ \sum_{i
    \neq k} n_{i,c}} \!
 + \ \frac{e_a}{\nvs} \cdot  \underbrace{
\frac{\partial n_a }{ \partial n_{k, c} } 
}_{= \ 0}
\ + \ e_a n_a \frac{\partial (1/\nvs) }{ \partial n_{k, c} } 
 \right\}  \\
  & = &
\frac{e_c}{\nvs} - \frac{ e_c n_c}{ \nvs^2 }
+  \sum_{a\neq c} \left\{ 
\sum_{i \neq k} n_{i,c} \left( \frac{n_a}{\nvs} \right)
- \frac{ e_a n_a}{ \nvs^2 }
 \right\}  \\
  & = &
\sum_{i \neq k, a\neq c} \frac{n_a n_{i,c} }{\nvs}
- \sum_{b} \frac{ e_b n_b}{ \nvs^2 } + \frac{e_c}{\nvs}
\end{eqnarray}
and thus
\begin{eqnarray}
\frac{\partial \theta}{\partial n_{k, c} }  = 
\sum_{i\neq k, a \neq c} n_{i,a} + \sum_{i \neq k, a\neq c} \frac{n_a
  n_{i,c} }{\nvs} + \frac{ e_c  } {\nvs } 
- \sum_{b} \frac{ e_b n_b  } {\nvs ^ 2 }. 
\end{eqnarray}
\end{proof}

\section{A Bad Example for the Info-Gain and Value of
  Information Criteria}
\label{sec:infogain-stinks}
\newcommand{\binary}[2]{\phi_{#1}\paren{#2}}
\newcommand{\testlin}[0]{\test^{\text{seq}}}
\newcommand{\testdum}[0]{\test^{\text{dumb}}}
\newcommand{\testbin}[0]{\test^{\text{bin}}}
\newcommand{\deltalin}[0]{\Delta^{\text{seq}}}
\newcommand{\deltabin}[0]{\Delta^{\text{bin}}}

A popular heuristic for the \ODT problem are to adaptively greedily
select the test that maximizes the 
\emph{information gain} in the distribution over hypotheses, conditioned on all previous test outcomes.
The same heuristic can be applied to the \ECD problem, in which we compute the information gain with respect to the entropy of the distribution over \emph{classes} rather than hypotheses.
Let $\igpolicy$ denote the resulting policy for \ECD.

Another common heuristic for \ODT is to adaptively greedily 
select the test maximizing the \emph{Bayesian decision-theoretic value
  of information} (VoI) criterion.
Recall the value of information of a test $\test$ is the expected
reduction in the expected risk of the minimum risk decision, where the
risk is the expected loss.  
Formally, consider the Bayesian
decision-theoretic setup described in \secref{sec:noise}.  
The VoI criterion myopically selects test to maximize 
$$\diff{VoI}{\prlzvec{\cA}}{\test} := \min_{\dec}
\expctover{\hvar}{\loss(\dec,\hvar)\mid \prlzvec{\cA}} - 
\expctover{\testval_{\test} \sim \testvar_{\test}\mid
  \prlzvec{\cA}}
{\min_{\dec} \expctover{\hvar}{\loss(\dec,\hvar)\mid \prlzvec{\cA}, \testval_{\test}}}.$$
This heuristic can be also be applied to the \ECD problem, by taking 
the decision set $\decset$ to be the set of equivalence classes, and 
the loss function to be the $0$--$1$ classification loss function, i.e., 
$\loss(\dec, \hvar) = \indicator{\hvar \notin \dec}$.
Let $\voipolicy$ denote the resulting policy.

In this section we present a family of \ECD instances for which both $\igpolicy$
and $\voipolicy$ perform significantly worse than the optimal policy.%

\begin{theorem} \label{thm:info-gain-stinks}
  There exists a family of \ECD instances with uniform priors such
  that $\cost(\igpolicy) = \Omega\paren{n / \log (n)}
  \cost(\policy^*)$ and $\cost(\voipolicy) = \Omega\paren{n / \log (n)}
  \cost(\policy^*)$, where $n$ is the number of hypotheses and
  $\policy^*$ is an optimal policy.
\end{theorem}

In fact, we will prove a lower bound for each policy within a large family of adaptive greedy
policies which contains $\igpolicy$ and $\voipolicy$, which we call
\emph{posterior--based}.
Informally, this family consists of all greedily policies that use
only information about the posterior equivalence
class distribution to select the next test. 
More precisely, these policies define a potential function $\pot$ which maps
distributions of distributions over equivalence classes to real
numbers, and at each time step select the test $\test$ which maximizes
the $\pot$ of the posterior distribution (over test outcomes $\outcome_{\test}$)
of the posterior distribution over equivalence classes generated by
adding  $\outcome_{\test}$ to the previously seen test outcomes.
In the event of a tie, we select any test maximizing this quantity at
random.  The information gain policy is posterior--based; $\pot$ is simply
$-1$ times the expected entropy of the posterior equivalence-class distribution.
Likewise, the value of information policy is also posterior--based; 
$\pot$ is simply
$-1$ times the expected loss of the best action for the posterior equivalence-class distribution.
Hence to prove \thmref{thm:info-gain-stinks} it suffices to prove the
following more general theorem.

\begin{theorem} \label{thm:potential-based-policies-stink}
  There exists a family of \ECD instances with uniform priors such 
  that $\cost(\policy) = \Omega\paren{n / \log (n)} \cost(\policy^*)$
  for any posterior--based policy $\policy$,  where $n$ is the number
  of hypotheses and
  $\policy^*$ is an optimal policy.
\end{theorem}

\begin{proof}
Fix integer parameter $q \ge 1$.
There are $\nclass = 2^q$ classes $\H_{a}$ for each 
$1 \le a \le 2^q$.  Each $\H_{a}$ consists of two hypotheses, $h_{a,0}$ and $h_{a,1}$.
We call $a$ the \emph{index} of $\H_{a}$.
The prior is uniform over the hypotheses $\hypotheses = \set{\hypothesis_{a,v} : 1 \le a \le \nclass, 0 \le v \le 1}$.
There are four types of tests, all with binary outcomes and all of unit cost.
There is only one test of the first type, $\test_0$, which tells us the value of $v$ in the realized hypothesis $\htrue_{a,v}$.
Hence for all $a$, $\hvar = \hypothesis_{a,v} \implies X_{\test_0} = v$.
Tests of the second type are designed to help us quickly discover the index of the realized class via binary search
if we have already run $\test_0$, but to offer no information gain whatsoever if $\test_0$ has not yet been run.
There is one such test $\test_k$ for all $t$ with $1 \le k \le q$.
For $z \in \nats$, let $\binary{k}{z}$ denote the $k^{\text{th}}$ least-significant bit of the binary encoding of $z$, so that $z = \sum_{k = 1}^{\infty} 2^{k-1} \binary{k}{z}$.
Then for each $\hypothesis_{a,v}$ we have 
$\hvar = \hypothesis_{a,v} \implies X_{\test_k} = \indicatorset{\binary{k}{a} = v}$.
Tests of the third type are designed to allow us to do a
(comparatively slow) sequential search
on the index of the realized class.
Specifically, we have tests $\testlin_k$ for all $1 \le k \le \nclass$, such
that 
$\hvar = \hypothesis_{a,v} \implies X_{\testlin_k} = \indicatorset{a = k}$.
Finally, tests of the fourth type, $\set{\testdum_k : k \in \nats}$ ,
are dummy tests that reveal no information at all.  Formally, $X_{\testdum_k}$ always equals
zero.

Given this input, suppose $\hvar = \hypothesis_{a,v}$.
One solution is to run $\test_0$ to find $v$, then run tests $\test_1, \ldots, \test_q$ to determine 
$\binary{k}{a}$ for all $1 \le k \le q$ and hence to determine $a$.  This reveals the value of $\hvar$, and hence the class $\hvar$ belongs to.  Since the tests have unit cost, this policy $\policy'$ has cost $\cost(\policy') = q+1$.

Next, fix a posterior--based policy $\policy$ and consider what it will do.  
Call a class \emph{possible} if not all of its hypotheses have been
ruled out by tests performed so far.
Note that all possible classes contain the same number of hypotheses,
because they initially have two, and each test $\test_k$ that can reduce the size of
a possible class to one, will reduce the size of
every possible class to one.  This, and the fact that the prior is
uniform, implies that the posterior equivalence-class distribution is
uniform over the remaining possible classes.
If no tests in $\set{\test_k : 0 \le k \le q}$ have been run, as is
initially the case, any single test in this set will not change the
posterior equivalence-class distribution.
Hence, as measured with respect to $\pot$, such tests are precisely as
good as the dummy tests.  
If these tests are each better than any test
in $\set{\testlin_k : 1 \le k \le \nclass}$, then $\policy$ selects
among $\set{\test_k : 0 \le k \le q} \cup
\set{\testdum_k :  k \in \nats}$ at random.  Since there are infinitely many
dummy tests, with probability one a dummy test is selected.
Since the posterior remains the same, $\policy$ will repeatedly select
a test at random from this set, resulting in an infinite loop as dummy
tests are selected repeatedly \emph{ad infinitum}.
Otherwise, some test $\testlin_k$ is
preferable to the other tests, measured  with respect to $\pot$.
In the likely event that $t$ is not the index of $\hvar$, 
we are left with a
residual problem in which tests in $\set{\test_k : 0 \le k \le q}$
still have no effect on the posterior, there is one
less class, and the prior is again uniform. 
Hence our previous argument still applies, and $\policy$ will
either enter an infinite loop or will repeatedly select tests in $\set{\testlin_k : 1 \le k \le \nclass}$ until
a test has an outcome of $1$.  Thus in expectation $\policy$ costs at least 
$\cost(\policy) \ge \frac{1}{\nclass}\sum_{z=1}^{\nclass} z  =
(\nclass+1)/2$.
Since $\nclass = 2^{q}$, $n =
2 \nclass$, and $\cost(\policy^*) \le \cost(\policy') = q+1 = \log_2(n)$ we infer 
$$\cost(\policy) \ge \frac{\nclass}{2} = \paren{\frac{n}{4 \log_2
    (n)}} \cost(\policy^*) $$
which completes the proof.
\end{proof}

} { }

\end{document}